\documentclass[journal]{IEEEtran} 
\usepackage{amsmath,amsfonts}
\usepackage{algorithmic}
\usepackage[linesnumbered,ruled,vlined]{algorithm2e}
\usepackage{array}
\ifCLASSOPTIONcompsoc
\usepackage[caption=false, font=normalsize, labelfont=sf, textfont=sf]{subfig}
\else
\usepackage[caption=false, font=footnotesize]{subfig}
\fi
\usepackage{textcomp}
\usepackage{stfloats}
\usepackage{verbatim}
\usepackage{graphicx}
\usepackage{cite}
\usepackage{bm,bbm}
\usepackage{amssymb}
\usepackage{booktabs}
\usepackage{threeparttable}
\usepackage{multirow}
\usepackage{amsthm}
\usepackage{accents}
\usepackage{color}
\usepackage{makecell}
\usepackage{xurl}

\newtheorem{lemma}{Lemma}[section]

\newtheorem{proposition}{Proposition}[section]

\newtheorem{remark}{Remark}[section]
\newtheorem{example}{Example}[section]
\newtheorem{assumption}{Assumption}[section]

\begin{document}

\title{Robust Predictive Motion Planning by Learning Obstacle Uncertainty}

\author{Jian Zhou,  Yulong Gao, Ola Johansson, Bj\"orn Olofsson, and Erik Frisk 
\thanks{This research was supported by the Strategic Research Area at Link\"oping-Lund in Information Technology (ELLIIT). (\it{corresponding author: Yulong Gao}).}
\thanks{Jian Zhou, Ola Johansson, and Erik Frisk are with the Division of Vehicular Systems, Department of Electrical Engineering, Link\"oping University,  SE-581 83 Link\"oping, Sweden. Email: \texttt{\{jian.zhou, ola.johansson, erik.frisk\}@liu.se}.
}
\thanks{Yulong Gao is with the Department of Electrical and Electronic Engineering,
Imperial College London, United Kingdom. Email: \texttt{yulong.gao@imperial.ac.uk}.}
\thanks{Bj\"orn Olofsson is with the Department of Automatic Control, Lund University, SE-221 00 Lund, Sweden, and also with the Division of Vehicular Systems, Department of Electrical Engineering, Link\"oping University, SE-581 83 Link\"oping, Sweden. Email: 
\texttt{bjorn.olofsson@control.lth.se}.
}
}

\markboth{This paper has been accepted for publication by IEEE Transactions on Control Systems Technology}%
{Shell \MakeLowercase{\textit{et al.}}: A Sample Article Using IEEEtran.cls for IEEE Journals}

\maketitle

\begin{abstract}
Safe motion planning for robotic systems in dynamic environments is nontrivial in the presence of uncertain obstacles, where estimation of obstacle uncertainties is crucial in predicting future motions of dynamic obstacles. The worst-case characterization gives a conservative uncertainty prediction and may result in infeasible motion planning for the ego robotic system. In this paper, an efficient, robust, and safe motion-planning algorithm is developed by learning the obstacle uncertainties online. More specifically, the unknown yet intended control set of obstacles is efficiently computed by solving a linear programming problem. The learned control set is used to compute forward reachable sets of obstacles that are less conservative than the worst-case prediction. Based on the forward prediction, a robust model predictive controller is designed to compute a safe reference trajectory for the ego robotic system that remains outside the reachable sets of obstacles over the prediction horizon. The method is applied to a car-like mobile robot in both simulations and hardware experiments to demonstrate its effectiveness.
\end{abstract}


\section{Introduction}\label{sec: Introduction}
\begin{figure}[t]
\centering
\subfloat[ ]{\includegraphics[width=0.7\columnwidth]{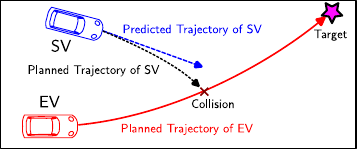}}
\vfil
\subfloat[ ]{\includegraphics[width=0.7\columnwidth]{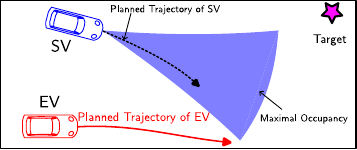}}
\vfil
\subfloat[ ]{\includegraphics[width=0.7\columnwidth]{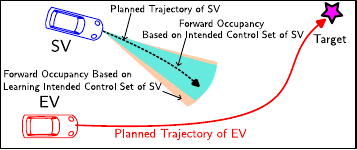}}
\centering
\caption{Reach-avoid scenario where the ego vehicle (EV) aims at planning a safe trajectory in the presence of an uncertain surrounding vehicle (SV). (a) Deterministic approach: no consideration of motion uncertainty of the obstacle, therefore the planning is unsafe. (b) Worst-case robust approach: considering the worst-case motion uncertainty of the obstacle yields an infeasible problem. (c) The proposed approach: reducing the conservatism and planning a feasible reference trajectory.}
\label{fig:Motivation_Example}
\end{figure}
\IEEEPARstart{S}{afe} motion planning is an important research topic in control and robotics \cite{zhou2024homotopic}, \cite{seo2022real}. Dynamic and uncertain environments make it challenging to perform efficient, safe and robust planning for a robotic system. One crucial source of uncertainty from the environment is obstacle uncertainty, which refers to the uncertain future motion of dynamic obstacles \cite{batkovic2023experimental}. This uncertainty can be exemplified by unknown decision-making of human-driven vehicles, which typically acts as uncertainties to autonomous vehicles \cite{gao2022risk}.

Predicting obstacle uncertainties plays a key role in predicting future motions of dynamic obstacles \cite{nair2022stochastic}. Although the worst-case realization of uncertainties in theory ensures safety of the ego robotic system, it results in overly conservative solutions or even infeasibility of the motion-planning problem \cite{gao2022risk}, \cite{fors2022resilient}, \cite{li2021prediction}. To reduce the conservativeness from the worst-case uncertainty characterization, this paper develops a safe motion-planning algorithm that learns the obstacle uncertainties using real-time observations of the obstacles. 

The motivation can be highlighted by the reach-avoid scenario in Fig.~\ref{fig:Motivation_Example}, where an autonomous ego vehicle (EV) navigates to reach the target while avoiding a surrounding vehicle (SV), e.g., a human-driven vehicle, in the driveable area. It is assumed that there is no direct communication between the two vehicles. Since the intended trajectory of the SV is unknown to the EV, the key point is how to predict the motion of the SV. A simple way for the prediction is by propagating a constant velocity model. Despite the computational advantage, the predicted trajectory of the SV may in this way have a deviation from its intended trajectory, as shown in Fig.~\ref{fig:Motivation_Example}(a). Thus, such predictions could be insufficient for the EV to perform safe motion planning. Alternatively, assuming that the EV knows the SV's dynamics and (possibly overestimated) admissible control set,  it is reasonable to predict the maximal forward occupancy that contains all trajectory realizations of the SV\footnote{Forward occupancy represents the set of positions a system can reach over a specific horizon, defining a region that could be occupied based on the system's control capabilities.}. Although uncertainties of the SV are incorporated in the prediction, the overly conservative forward occupancy may make the motion-planning problem infeasible, as illustrated in Fig.~\ref{fig:Motivation_Example}(b). Therefore, the primary objective of this paper is to refine the robust prediction method to reduce its conservativeness while maintaining safety. This involves understanding the intended control set of the SV, which, though typically unknown, can be learned from data sampled by observations of SV. By computing the forward occupancy of the SV based on learning its intended control set, the EV can plan a trajectory that is both safe and feasible, as shown in Fig.~\ref{fig:Motivation_Example}(c). 

To realize the objective, this paper considers the motion planning of a general ego robotic system in the presence of dynamic surrounding obstacles with control strategies unknown to the ego system. A new efficient way is proposed to learn the decision-making support, i.e., the set of intended control actions of each obstacle, using observations of obstacles. Thus, the use of this set enables a better motion prediction, in the sense of less conservative than the worst-case prediction. The main contributions of this paper are:
\begin{itemize}
\item [(a)] A novel approach is proposed to learn the unknown yet intended control set of obstacles without making any assumptions about the distribution of control actions of the obstacles and without the need for training with prior data. The set is efficiently computed by solving a linear programming (LP) problem.
\item [(b)] Using the online learned set, a robust predictive motion planner is developed for motion planning of the ego system subject to collision avoidance with uncertain surrounding obstacles in dynamic environments. The method can perform resilient motion planning in the absence of prior knowledge regarding obstacle uncertainties.
\item [(c)]  The performance of the proposed motion-planning method is validated through both simulations and hardware experiments in several benchmark traffic scenarios. 
\end{itemize}
\paragraph*{Outline} This paper is organized as follows: Section~\ref{sec: related work} places this research in the context of related work in the literature. Section~\ref{Problem Statement} formalizes the general research problem. Section~\ref{Quantification} specifies the method for learning the control set of obstacles. Section~\ref{General Robust MPC for Motion Planning} designs the robust MPC for safe motion planning based on uncertainty prediction. Sections~\ref{sec: simulations} and \ref{sec: hardware exp} show the performance of the method in both simulations and hardware experiments, and Section~\ref{Conclusion} concludes the paper.

\section{Related Work}\label{sec: related work}
This section reviews the most related research from two perspectives: (1) Uncertainty prediction of dynamic systems, and (2) uncertainty-aware motion planning in time-varying environments. 

\subsection{Uncertainty Prediction}\label{sec: uncertainty prediction}
The robust approach for predicting uncertainties of a dynamic system is by formulating a sequence of sets over the horizon to cover the worst-case uncertainty realizations. This is typically achieved by forward reachability analysis \cite{chen2018hamilton}, \cite{carrizosa2024safe}. For example, \cite{seo2022real} characterized the forward reachable sets (FRSs) of the ego robotic system considering the worst-case disturbances for safety guarantees in robust motion planning. In \cite{althoff2014online}, the FRSs were predicted for both an autonomous ego vehicle and a surrounding vehicle based on the maximal disturbance set for safety verification. Focusing on uncertainty predictions of other traffic participants, the later work in \cite{althoff2016set} and \cite{wang2024safe} combined the worst-case FRSs with the road networks and the interactions with the autonomous ego vehicle, respectively, in practical driving scenarios. An alternative yet more conservative approach to the forward reachability analysis is computing the minimal robust positive invariant (MRPI) set for the ego system under the maximal disturbances, and relevant works can be found in \cite{danielson2020robust}, \cite{dixit2020trajectory}, and \cite{nezami2022safe}. 

A main problem of robust methods is that the worst-case assumptions can lead to unnecessary conservatism in the prediction. To mitigate this problem, the non-conservative methods, which formulate the set of uncertainties as a subset of the worst-case set, have been developed. For instance, \cite{althoff2009model} and \cite{wang2023reachability} formulated the probabilistic reachable set of other vehicles in traffic by assuming the decision-making process of surrounding vehicles, which is influenced by the road geometry, as a Markov chain. If sufficient prior knowledge of obstacle uncertainties is available, e.g., a large amount of training data, the probabilistic prediction can be performed by machine learning models like neural ordinary differential equations \cite{westny2023mtp}.  The training data can also be utilized to formulate the empirical reachable set that contains the $\alpha$-likely observed trajectories of a system \cite{driggs2018robust}. In addition, with particular assumptions on the distribution of obstacles, a confidence-aware prediction that formulated an obstacle occupancy with a degree of prediction confidence/risk was designed in \cite{fridovich2020confidence} and \cite{gao2022risk}, where the former assumed that the control action of an obstacle satisfies a Boltzmann distribution, and the later assumed that the leading vehicle in the automated overtaking scenario respects a supermartingale. Apart from the probabilistic prediction, learning the true disturbance set of a dynamic system was also investigated to reduce the conservativeness, as presented in \cite{gillula2013reducing} that learned the disturbance set of an ego helicopter using residuals, and \cite{gao2023robust} that learned the disturbance set of an ego vehicle using a rigid tube. 

\subsection{Uncertain-Aware Motion Planning}
Based on the prediction results, the motion-planning problems in uncertain and dynamic environments can be solved using various well-established methods \cite{paden2016survey}. Among them, model predictive control (MPC) stands out because of the straightforward inclusion of differential constraints of the system and safety constraints with uncertain obstacles and the environment \cite{benciolini2023non}. Combined with different uncertainty-prediction methods, the MPC-based planners include robust MPC, stochastic MPC, and scenario MPC \cite{zhou2022interactionnew}. For example, in \cite{batkovic2023experimental} and \cite{pek2021fail}, the robust MPC was used to generate the safe reference trajectory for an autonomous vehicle subject to the worst-case obstacle occupancy. In contrast, the stochastic MPC, as presented in \cite{nair2022collision} and \cite{brudigam2023stochastic}, considers collision avoidance with obstacle occupancy with a probability based on the assumptions on the uncertainty distribution of the obstacle. The scenario MPC can also achieve similar probabilistic constraints, where the constraints are transformed into deterministic forms using finite samples of uncertainties \cite{de2021scenario}. In addition to these methods, there have been combinations of them, like the robust scenario MPC that optimizes policies concerning multi-mode obstacle-uncertainty predictions \cite{batkovic2021robust}, and the scenario stochastic MPC to also handle multi-modal obstacle uncertainties with known distributions \cite{brudigam2021collision}.

\section{Problem Statement}\label{Problem Statement}
This section first defines the general notations used in the paper, then introduces models of the ego system and dynamic surrounding obstacles and finally formulates the robust motion-planning problem to be solved. 

\subsection{Notations}\label{sec: Notation}
$\mathbb{R}^n$ is the $n$-dimensional real number vector space, $\mathbb{R}_{+}^n$ is the $n$-dimensional non-negative vector space, $\mathbb{N}^n$ is the $n$-dimensional natural number vector space, and $\mathbb{N}_{+}^n$ is the $n$-dimensional positive integer vector space. $[0, 1]^n$ means an $n$-dimensional vector with each element between $[0, 1]$. $I^n$ indicates an $n \times n$ identity matrix. Matrices of appropriate dimension with all elements equal to 1 and 0 are denoted by $\bm{1}$ and $\bm{0}$, respectively. An interval of integers is denoted by $\mathbb{I}_a^b = \{a, a+1, \cdots, b\}$. For two sets $\mathcal{A}$ and $\mathcal{B}$, $\mathcal{A} \oplus \mathcal{B} = \left\{ x + y \ | \ x \in \mathcal{A}, y \in \mathcal{B} \right\}$. For a set $\Gamma \subset \mathbb{R}^{n_a+n_b}$, the set projection is defined as $\mathcal{A} = {\rm Proj}_a(\Gamma)=\left\{ a\in \mathbb{R}^{n_a}: \exists b\in \mathbb{R}^{n_b}, (a, b)\in \Gamma \right\}$. For the MPC iteration, the current time step is indicated by $t$. A prediction of a variable $x$ at time step $t+i$ over the prediction horizon is represented as $x_{i|t}$, where $i = 1, \ldots, N$, and $N$ is the prediction horizon. The sampling interval is represented as $T$. The cardinality of a set $\mathcal{A}$ is represented by $|\mathcal{A}|$. The weighted inner product of a vector $x$ is denoted by $||x||_A^2 = x^{\top}Ax$.

\subsection{Modeling of Ego System and Surrounding Obstacles}\label{Sec:Modeling}
The ego system's dynamics are described by a discrete-time control system
\begin{equation}
x^e_{t+1} = f(x^e_{t}, \ u^e_{t}), \label{eq_ego_model}
\end{equation}
\noindent where $x_{t}^e \in \mathbb{X}^e\subseteq \mathbb{R}^{n^e_x}$ is the state and $u_{t}^e \in \mathbb{U}^e\subseteq \mathbb{R}^{n^e_u}$ is the control input at time step $t$.  Here, the superscript $e$ indicates the ego system. The sets $\mathbb{X}^e$ and $\mathbb{U}^e$ are the state set and admissible control set, respectively. 

Consider an environment with multiple surrounding obstacles, where the motion of the center of geometry of each obstacle is modeled by a linear time-varying (LTV) system
\begin{equation}\label{eq_obs_LTV_model}
x^s_{t+1} = A^s_tx^s_{t} + B^s_{t}u^s_{t},
\end{equation}
\noindent where $A^s_{t}$ and $B^s_{t}$ are the system matrices at time step $t$, which are determined according to specific kinematic or dynamic equations of the obstacle. The vector $x^s_{t}\in  \mathbb{R}^{n^s_x}$ is the state and $u^s_{t}\in  \mathbb{R}^{n^s_u}$ is the input at time step $t$. The superscript $s$ indicates the $s$-th obstacle and the index set of all involved obstacles is denoted by a finite set $\mathcal{S}$.  

As illustrated in the motivating example in Section~\ref{sec: Introduction}, the control capability of the obstacles is distinguished by two different control sets. For the $s$-th obstacle, {\it the admissible control set} $\mathbb{U}^s$ refers to the set of control inputs determined by the physical limitations in the worst case, while {\it the intended control set} $\tilde{\mathbb{U}}_t^s\subseteq \mathbb{U}^s$ refers to the set that covers the control inputs that would be used by the obstacle from step $t$ over a planning horizon.  It is assumed that  $\mathbb{U}^s$  is known, or can be overly estimated, by the ego system, while $\tilde{\mathbb{U}}_t^s$ is unknown.

The sets $\tilde{\mathbb{U}}_t^s$ and $\mathbb{U}^s$ can be illustrated by the following example. The maximal acceleration of a vehicle at the limit of friction could be $\pm 8 \ {\rm m/s^2}$. In practice, because of the capacity of the power system or caution of the driver, the vehicle may only generate accelerations within $\pm 2 \ {\rm m/s^2}$. In this case, the admissible control set is $\mathbb{U}^s=[-8, 8] \ {\rm m/s^2}$, while the intended control set is $\tilde{\mathbb{U}}_t^s=[-2, 2] \ {\rm m/s^2}$, $\forall t\in \mathbb{N}$.

Example~\ref{Ex:examplemodel} will present the models of the ego system and surrounding obstacles in a specific scenario to illustrate the general forms in \eqref{eq_ego_model} and \eqref{eq_obs_LTV_model}.

\begin{remark}\label{remark:linear model}
The LTV model \eqref{eq_obs_LTV_model} is used for two main reasons. (1) Since the true dynamics of the obstacle are unknown, a linear model can easily be derived to approximate the motion characteristics of the obstacle. It benefits the computation of forward reachable sets when predicting the motion of obstacles. (2)  The linear model also facilitates the estimation of the control inputs of obstacles for the proposed method, while this can be challenging or computationally expensive for nonlinear models.
\end{remark}

\begin{example}\label{Ex:examplemodel}
{\textbf{Vehicle models in the reach-avoid problem:}} Consider the reach-avoid example in Fig.~\ref{fig:Motivation_Example} throughout the paper, the EV in Fig.~\ref{fig:Motivation_Example} is modeled by a nonlinear single-track kinematic model \cite{brudigam2023stochastic},
\begin{subequations}\label{eq: EV model}
\begin{align}
\dot{p}_x^e & = v^e{\rm cos}\left(\varphi^e + \beta^e \right), \label{eq_EV_model_a} \\
\dot{p}_y^e & = v^e{\rm sin}\left(\varphi^e + \beta^e\right), \label{eq_EV_model_b} \\
\dot{\varphi}^e & = \frac{v^e}{l_r}{\rm sin}\left(\beta^e\right), \label{eq_EV_model_c} \\
\dot{v}^e & = a^e, \label{eq_EV_model_d} \\
\dot{a}^e & = \eta^e, \label{eq_EV_model_e} \\
\beta^e & = {\rm arctan}\left(\frac{l_r}{l_f + l_r}{\rm tan}(\delta^e)\right), \label{eq_EV_model_f}
\end{align}
\end{subequations}
\noindent where $(p^e_x, p^e_y)$ is the coordinate of the center of geometry of the vehicle in the ground coordinate system. The variable $\varphi^e$ is the yaw angle, $\beta^e$ is the angle of the velocity of the center of geometry with the longitudinal axis of the vehicle, $v^e$, $a^e$, and $\eta^e$ mean the longitudinal speed, acceleration, and jerk, in the vehicle frame, respectively, $\delta^e$ is the front tire angle, and parameters $l_f$ and $l_r$ are the distances from the center of geometry to the front axle and rear axle, respectively. The model takes $a^e$ and $\eta^e$ as inputs and the other variables as states. This continuous-time model can be approximated to a discrete-time model in \eqref{eq_ego_model} using a sampling interval $T$ with different discretization methods \cite{ascher1998computer}.

The actual model of the SV in Fig.~\ref{fig:Motivation_Example} is unknown. From the perspective of the EV, the motion of the center of geometry of the SV is modeled by a linear time-invariant model 
\begin{equation}
x_{t+1}^s = A^sx_t^s + B^su_t^s, \label{eq_SV_model}
\end{equation}
where
$$x_t^s = [p_{x, t}^s \ v_{x, t}^s \ p_{y, t}^s \ v_{y, t}^s]^{\top}, \ u_t^s = [a_{x,t}^s \ a_{y,t}^s]^{\top},$$
$$A^s = \begin{bmatrix} 1 & T & 0 & 0\\ 0 & 1 & 0 & 0 \\0 & 0 & 1 & T \\ 0 & 0 & 0 & 1 \end{bmatrix}, \ B^s = \begin{bmatrix} T^2/2 & 0 \\ T & 0 \\ 0 & T^2/2 \\ 0 & T \end{bmatrix},$$
and the variables $p_{x, t}^s$, $v_{x, t}^s$, $a_{x, t}^s$ mean the longitudinal position, velocity, and acceleration at time step $t$, respectively, in the ground coordinate system. Variables $p_{y, t}^s$, $v_{y, t}^s$, and $a_{y, t}^s$ mean the lateral position, velocity, and acceleration at time step $t$, respectively, in the ground coordinate system. The control input satisfies $u_t^s \in \tilde{\mathbb{U}}_t^s \subseteq \mathbb{U}^s$, where the details will be given in case studies in Section~\ref{sec: simulations}. 

\end{example}

\subsection{Research Problem}\label{Problem}
The goal of the ego system is to plan an optimal and safe reference trajectory to fulfill desired specifications in the presence of uncertain surrounding obstacles.  The planned reference trajectory is anticipated to be less conservative than that of the worst-case realization, yet robust against the obstacles' uncertain motions. Toward this goal, this paper has to solve three sub-problems: (i) learning the unknown intended control set $\tilde{\mathbb{U}}_t^s$ of each obstacle, i.e., computing a set to approximate $\tilde{\mathbb{U}}_t^s$; (ii) predicting the forward occupancy of each obstacle based on the obstacle model \eqref{eq_obs_LTV_model} and the learned intended control set of the obstacle; (iii) planning an optimal reference trajectory that enables the ego system to avoid collision with the occupancy of obstacles and satisfy differential and admissible constraints on the ego system model \eqref{eq_ego_model}. The first two sub-problems are addressed in Section~\ref{Quantification}, and the motion-planning method is presented in Section~\ref{General Robust MPC for Motion Planning}.

\section{Learning Obstacle Uncertainty} \label{Quantification}
This section describes how to learn the intended control set $\tilde{\mathbb{U}}_t^s$ of a general surrounding obstacle by a sampling-based approach, and how to predict the forward occupancy for the obstacle based on learning the set $\tilde{\mathbb{U}}_t^s$. 

\subsection{Learning Intended Control Set $\tilde{\mathbb{U}}_t^s$} \label{Quantification of the Precise Input Set of the Obstacle}
\begin{assumption}\label{assum:measure control}
The system~\eqref{eq_obs_LTV_model} does not have redundant actuators, and at the current time step $t$, the state $x_t^s$ can be measured or estimated by the ego system.
\end{assumption}

This implies that the matrix $B_{t-1}^s$ of model~\eqref{eq_obs_LTV_model} has full column rank, such that at the current time step $t$, the ego system can find unique solution when solving for $u_{t-1}^s$, i.e., the control input of \eqref{eq_obs_LTV_model} at the previous time step $t-1$, from the expression $B_{t-1}^su_{t-1}^s = x_{t}^s - A_{t-1}^sx_{t-1}^s$. For example, the solution $u_{t-1}^s$ in~\eqref{eq_SV_model} is $u_{t-1}^s = [(v_{x,t}^s - v_{x,t-1}^s)/T \ (v_{y,t}^s - v_{y,t-1}^s)/T]^{\top}$. Since the intended control set is a hidden variable, an approach is to infer the hidden information using past observations~\cite{hu2023active}. To this end, based on Assumption~\ref{assum:measure control}, we define the information set $\mathcal{I}^s_t$ available to the ego system at time step $t$, which consists of the collected inputs of the $s$-th obstacle until time step $t$. The information set is updated as
\begin{equation}
\mathcal{I}^s_{t}=\{u^s_k: k \in \mathbb{I}_0^{t-1}\}. \label{eq_update_information}
\end{equation}

The set $\mathcal{I}_t^s \subseteq \tilde{\mathbb{U}}_t^s$ contains observed control actions of obstacle $s$, and the process of generating the control actions by the obstacle is unknown to the ego system. Here, it is modeled as a random process such that the elements of the set $\mathcal{I}_t^s$ can be regarded as samples generated with unknown distributions from the intended control set of the obstacle. Next, the process of learning the unknown set $\tilde{\mathbb{U}}_t^s$ using the information set $\mathcal{I}^s_t$ and the known admissible control set $\mathbb{U}^s$ will be demonstrated. The following assumption is made before proceeding. 
\begin{assumption}\label{assum:poly_U}
For any obstacle $s$, the admissible control set $\mathbb{U}^s$ is a convex and compact polytope that contains the origin in its interior, denoted by 
\begin{equation*}
\mathbb{U}^s = \left\{u \in \mathbb{R}^{n^s_u}: H^su \leq \bm{1}\right\}, \ H^s \in \mathbb{R}^{n_v^s \times n_u^s}. \label{eq_poly_U}
\end{equation*}
\end{assumption}

In order to learn the unknown set $\tilde{\mathbb{U}}_t^s\subseteq \mathbb{U}^s$, a subset of $\mathbb{U}^s$ is parameterized in the form of
\begin{subequations}
\begin{align}
\mathbb{U}^s(v, \rho, \bm{\theta}) &= f(H^s, \bm{\theta}) \oplus \{(1-\rho)v\}, \label{Eq:quanset_a} \\
f(H^s, \bm{\theta}) &= \left\{u \in \mathbb{R}^{n_u^s}: H^su \leq \bm{\theta}\right\}, \label{Eq:quanset_b}
\end{align}
\end{subequations}
\noindent where $v \in \mathbb{U}^s$, $\rho \in [0, 1]$, and $\bm{\theta} \in [0, 1]^{n_v^s}$ are design parameters. It is shown in Lemma~\ref{lemma: subset} that such parameterization under Assumption~\ref{assum:poly_U} constructs a subset of $\mathbb{U}^s$.

\begin{lemma} \label{lemma: subset}
For any $v \in \mathbb{U}^s$, $\bm{\theta} \in [0, 1]^{n_v^s}$, and $\rho \in [0, 1]$, if $\bm{\theta} \leq \rho \bm{1}$, then
$\mathbb{U}^s(v, \rho, \bm{\theta}) \subseteq \mathbb{U}^s$.
\end{lemma}
\begin{proof}
The explicit form of \eqref{Eq:quanset_a}--\eqref{Eq:quanset_b} can be written as 
\begin{equation*} \label{Eq: explicit_uncertainty_quan}
\mathbb{U}^s(v, \rho, \bm{\theta}) = \left\{u \in \mathbb{R}^{n^s_u}: H^su \leq \bm{\theta} + (1-\rho)H^sv\right\}.
\end{equation*}

It follows that $\bm{\theta} + (1-\rho)H^sv \leq \bm{\theta} + (1-\rho)\bm{1} \leq \bm{1}$. Therefore, $\mathbb{U}^s(v, \rho, \bm{\theta}) \subseteq \mathbb{U}^s$.
\end{proof}

By parameterizing the set $\mathbb{U}^s(v, \rho, \bm{\theta})$ using \eqref{Eq:quanset_a}--\eqref{Eq:quanset_b} and leveraging its properties shown in Lemma~\ref{lemma: subset}, it is possible to find the optimal parameters $\rho^{\star}$, $v^{\star}$, and ${\bm \theta^{\star}}$ that minimize the volume of set $\mathbb{U}^s(v, \rho, \bm{\theta})$, while containing the historical information set $\mathcal{I}_t^s$. This optimal set, which is denoted by $\hat{\mathbb{U}}_t^s$, serves as the optimal approximation of the intended control set at time step $t$. The optimal parameters $v^{\star}$, $\rho^{\star}$, and ${\bm \theta^{\star}}$ can be obtained by 
solving the following problem:
\begin{equation} \label{Opt2:quanset}
\begin{aligned}
\mathop{\rm minimize}\limits_{v, \rho, \bm{\theta}} \quad & \bm{1}^{\top}\bm{\theta} + \rho & \\
{\rm subject \ to}\quad
&H^s u^s-(1-\rho)H^sv \leq \bm{\theta}, \ \forall  u^s\in \mathcal{I}_t^s, & \\
&H^sv \leq \bm{1}, & \\
&\rho \in [0, 1],  \ \bm{\theta} \in [0, 1]^{n_v^s}, \ \bm{\theta} \leq \rho\bm{1}, &
\end{aligned}
\end{equation}
\noindent where the cost function minimizes the volume of set $\mathbb{U}^s(v, \rho, \bm{\theta})$ by minimizing summation of the vector ${\bm \theta}$ and its upper bound $\rho$. The first constraint of \eqref{Opt2:quanset} implies that $\mathcal{I}_t^s \subseteq \mathbb{U}^s(v, \rho, \bm{\theta})$, i.e., the information set $\mathcal{I}_t^s$ is fully contained within the parameterized set $\mathbb{U}^s(v, \rho, \bm{\theta})$. The second and third constraints are required by Lemma~\ref{lemma: subset}. The problem \eqref{Opt2:quanset} is non-convex as a result of the product of decision variables $v$ and $\rho$. The following proposition shows that \eqref{Opt2:quanset} is equivalent to a linear programming (LP) problem.

\begin{proposition} \label{LP for quanset}
Define an LP problem as
\begin{equation} \label{LP:quanset}
\begin{aligned}
\mathop{\rm minimize}\limits_{y, \rho, \bm{\theta}} \quad & \bm{1}^{\top}\bm{\theta} + \rho & \\
{\rm subject \ to}\quad
&H^s u^s-H^sy \leq \bm{\theta}, \ \forall  u^s\in \mathcal{I}_t^s, & \\
&H^sy \leq (1-\rho)\bm{1}, & \\
&\rho \in [0, 1],  \ \bm{\theta} \in [0, 1]^{n_v^s}, \ \bm{\theta} \leq \rho\bm{1}. &
\end{aligned}
\end{equation}

Denote the optimal solution of \eqref{LP:quanset} by $(y^{\star}, \rho^{\star}, \bm{\theta}^{\star})$. Then, the optimal solution of \eqref{LP:quanset} gives an optimal estimation of the unknown intended control set as 
\begin{equation}\label{Eq:optquanset}
\hat{\mathbb{U}}_t^s = f(H^s, \bm{\theta}^{\star}) \oplus y^{\star},
\end{equation}
where $f(H^s, \bm{\theta}^{\star})$ is defined as in \eqref{Eq:quanset_b}. The set $\hat{\mathbb{U}}_t^s$ obtained in \eqref{Eq:optquanset} is equivalent to the set obtained by solving \eqref{Opt2:quanset}.
\end{proposition}
\begin{proof}
For \eqref{Opt2:quanset} with $0 \leq \rho < 1$, it is equivalent to \eqref{LP:quanset} with $0 \leq \rho < 1$. This equivalence arises from substituting the variable $y$ in \eqref{LP:quanset} with $(1-\rho)v$. When $\rho = 1$, \eqref{Opt2:quanset} transforms into:
\begin{equation} \label{Opt2:quanset_rho_1}
\begin{aligned}
\mathop{\rm minimize}\limits_{\rho, \bm{\theta}} \quad & \bm{1}^{\top}\bm{\theta} + \rho & \\
{\rm subject \ to}\quad
&H^s u^s\leq \bm{\theta}, \ \forall  u^s\in \mathcal{I}_t^s, & \\
&\rho = 1, \ \bm{\theta} \in [0, 1]^{n_v^s}, &
\end{aligned}
\end{equation}
\noindent where in \eqref{Opt2:quanset_rho_1}, the constraint $H^sv \leq \bm{1}$ from \eqref{Opt2:quanset} is removed as it is redundant. On the other hand, the equivalent form of \eqref{LP:quanset} with $\rho = 1$ is given by:
\begin{equation} \label{LP:quanset_rho_1}
\begin{aligned}
\mathop{\rm minimize}\limits_{y, \rho, \bm{\theta}} \quad & \bm{1}^{\top}\bm{\theta} + \rho & \\
{\rm subject \ to}\quad
&H^s u^s-H^sy \leq \bm{\theta}, \ \forall  u^s\in \mathcal{I}_t^s, & \\
&H^sy \leq \bm{0}, & \\
&\rho = 1,  \ \bm{\theta} \in [0, 1]^{n_v^s}. &
\end{aligned}
\end{equation}

The following shows that the constraint $H^sy \leq \bm{0}$ in \eqref{LP:quanset_rho_1} is satisfied if and only if $y = \bm{0}$.

Since the origin is an interior of the set $\mathbb{U}^s$, suppose $H^sy \leq \bm{0}$ and $y \neq \bm{0}$ hold. This means $\forall \varepsilon \in \mathbb{R}_+$ such that $H^s(\varepsilon y) \leq \bm{0} \leq \bm{1}$, this further implies that $\varepsilon y \in \mathbb{U}^s$, where $\varepsilon y$ is the terminal of a line segment starting from the origin. However, since the set $\mathbb{U}^s$ is bounded, increasing the value of $\varepsilon$ will eventually cause $\varepsilon y$ to exceed the limit of $\mathbb{U}^s$. Therefore, there does not exist any $y \neq {\bm 0}$ satisfying $H^sy \leq {\bm 0}$, and the only feasible solution of the variable $y$ in \eqref{LP:quanset_rho_1} is $y=\bm{0}$.

This confirms that \eqref{Opt2:quanset_rho_1} and \eqref{LP:quanset_rho_1} are equivalent with $\rho = 1$. Hence, we conclude that \eqref{Opt2:quanset} and \eqref{LP:quanset} are equivalent, and the optimal quantified set $\hat{\mathbb{U}}_t^s$ is given in the form of \eqref{Eq:optquanset}.
\end{proof}

The sets $\mathbb{U}^s$, $\tilde{\mathbb{U}}_t^s$, and $\hat{\mathbb{U}}_t^{s}$ are illustrated in Fig.~\ref{fig:set_illustration}, which reflects that the set $\hat{\mathbb{U}}_t^s$ depends on the shape of $\mathbb{U}^s$ and the information set $\mathcal{I}_t^s$, and does not relate to the volume of $\mathbb{U}^s$. 
\begin{remark}\label{remark: different quantification}
Formulation \eqref{LP:quanset} is a standard randomized program \cite{campi2008exact}. It is shown in \cite[Theorem 3]{gao2023robust} that if $\tilde{\mathbb{U}}_t^s$ is time-invariant, the sequence of control actions are independent and identically distributed, and $|\mathcal{I}_t^s|$ is large enough, then with high confidence, the set $\hat{\mathbb{U}}_t^s$ is a good approximation of the true intended set $ \tilde{{\mathbb{U}}}_t^s$. While these assumptions are stringent in practice, the results highlight that the robustness of $\hat{\mathbb{U}}_t^s$ is enhanced with a larger information set $\mathcal{I}_t^s$.
\end{remark}
\begin{figure}[!t]
\centering
\includegraphics[width=0.5\columnwidth]{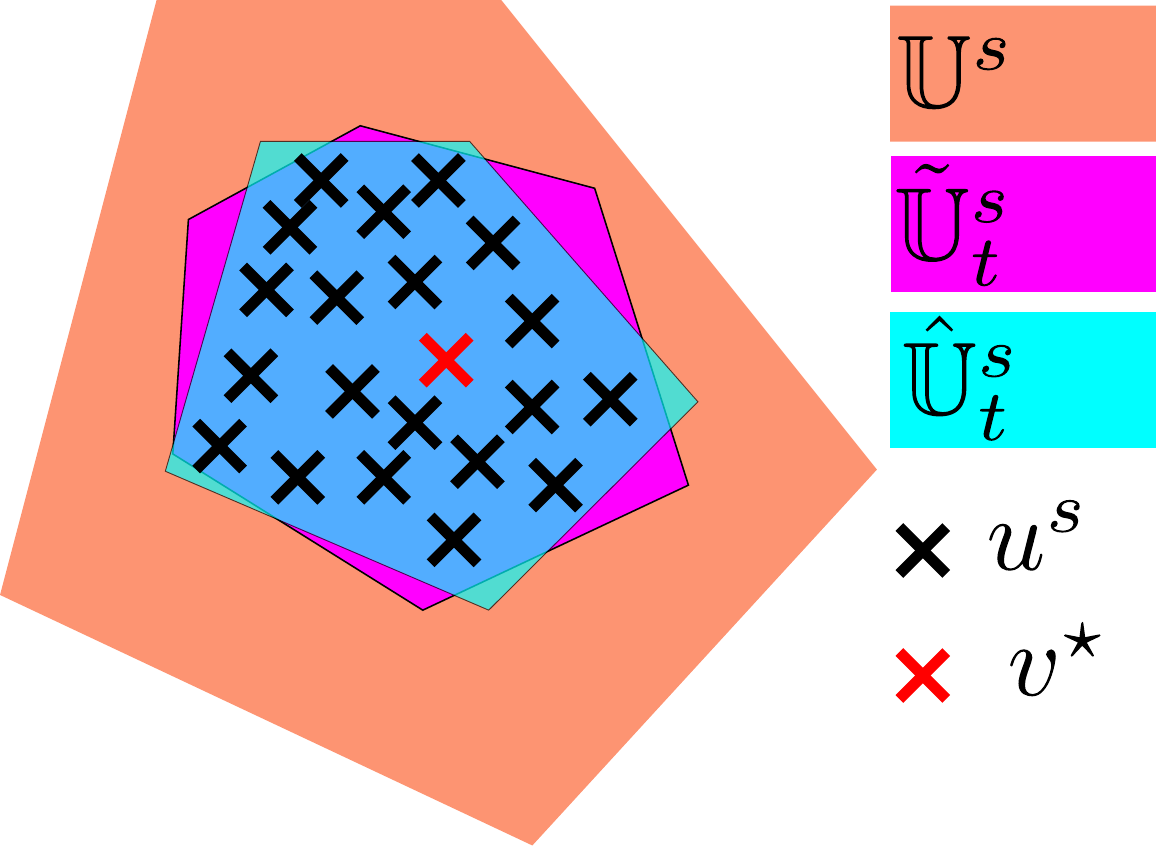}
\caption{Learning the set $\tilde{\mathbb{U}}_t^s$ by solving LP \eqref{LP:quanset}.}
\label{fig:set_illustration}
\end{figure}

\subsection{Computing $\hat{\mathbb{U}}_t^s$ with Low Complexity} \label{sec:Computing set with low complexity}
Note that the computational complexity of the LP \eqref{LP:quanset} for computing $\hat{\mathbb{U}}_t^s$ is polynomial in $|\mathcal{I}_t^s|$ if $\mathcal{I}_t^s$ is increasing with time. This means the complexity of \eqref{LP:quanset} increases with each time step. To address this problem, two ways are proposed to manage the online computational complexity. 

\subsubsection{Moving-Horizon Approach}\label{sec: moving horizon approach}
The first way is to control the size of $\mathcal{I}_t^s$ in a moving-horizon manner. That is, given a horizon $L \in \mathbb{N}_{+}$, the set $\mathcal{I}_t^s$ is updated as 
$$\mathcal{I}^s_{t}=\{u^s_k: k \in \mathbb{I}_{t-L-1}^{t-1}\}$$ 
with an initialization $|\mathcal{I}^s_{0}|=L$. In this way, \eqref{LP:quanset} is accompanied by $L$ scenario constraints.  

\subsubsection{Online Recursion}\label{sec: online recursion}
A second way is to perform an online recursive computation of  $\hat{\mathbb{U}}_t^s$ based on $\hat{\mathbb{U}}_{t-1}^s$ and $u^s_{t-1}$. This is achieved by solving the following problem  
\begin{equation} \label{Eq:quansetonline}
\begin{aligned}
\mathop{\rm minimize}\limits_{v, \rho, \bm{\theta}} \quad & \bm{1}^{\top}\bm{\theta} + \rho & \\
{\rm subject \ to}\quad
&u_{t-1}^s \in {\mathbb{U}}^s(v, \rho, \bm{\theta}), & \\
&\hat{\mathbb{U}}^s_{t-1} \subseteq \mathbb{U}^s(v, \rho, \bm{\theta}), & \\
&v \in \mathbb{U}^s,& \\
&\rho \in [0, 1], \ \bm{\theta} \in [0, 1]^{n_v^s}, \ \bm{\theta} \leq \rho\bm{1}. &
\end{aligned}
\end{equation}

The optimal solution $(v^{\star}, \rho^{\star}, \bm{\theta}^{\star})$ of \eqref{Eq:quansetonline} enables update of $\hat{\mathbb{U}}_t^s$, which is the minimum set in the form of \eqref{Eq:quanset_a} that covers the previous obtained set $\hat{\mathbb{U}}^s_{t-1}$ and the latest information sample $u_{t-1}^s$. Following the proof of Proposition~\ref{LP for quanset}, with $y=v(1-\rho)$, the optimal solution of the following LP provides the optimal estimation of the set $\tilde{\mathbb{U}}_t^s$ in the form of \eqref{Eq:optquanset}:
\begin{equation} \label{LP:quansetonline}
\begin{aligned}
\mathop{\rm minimize}\limits_{y, \rho, \bm{\theta}} \quad & \bm{1}^{\top}\bm{\theta} + \rho & \\
{\rm subject \ to}\quad
&H^s u_{t-1}^s-H^sy \leq \bm{\theta},  & \\
&H^sy^{\star}_{\rm pre} + \bm{\theta}^{\star}_{\rm pre} \leq H^sy + \bm{\theta},& \\
&H^sy \leq (1-\rho)\bm{1}, & \\
&\rho \in [0, 1],  \ \bm{\theta} \in [0, 1]^{n_v^s}, \ \bm{\theta} \leq \rho\bm{1}, &
\end{aligned}
\end{equation}
\noindent where $y^{\star}_{\rm pre}$ and $\bm{\theta}^{\star}_{\rm pre}$ are the optimal solution of \eqref{LP:quansetonline} at time step $t-1$, i.e., $\hat{\mathbb{U}}^s_{t-1}=f(H^s, \bm{\theta}_{\rm pre}^{\star}) \oplus y_{\rm pre}^{\star}$. The formulation \eqref{LP:quansetonline} is computationally efficient as the number of constraints does not change with $|\mathcal{I}_t^s|$, and it includes all samples in the set $\mathcal{I}_t^s$. The online recursion~\eqref{LP:quansetonline} is initialized with $\hat{\mathbb{U}}_0^s$, which is obtained by solving~\eqref{LP:quanset} with an initial information set $\mathcal{I}_0^s$. The set can be formulated as a small, yet nonempty, set with artificial samples.

The characteristics of the moving-horizon approach and online recursion approach will be demonstrated by a set of numerical examples in Section~\ref{sec: running examples}.

\subsection{Forward Reachability Analysis Based on $\hat{\mathbb{U}}_t^s$} \label{Computation of Occupancy of the Obstacle}
Given $\hat{\mathbb{U}}_t^s$ obtained in Section~\ref{sec:Computing set with low complexity}, and the system described as in \eqref{eq_obs_LTV_model}, forward reachability can be used to predict the obstacle occupancy. Denote by $\hat{\mathcal{O}}_{i|t}^s$ the occupancy of the center of geometry of obstacle~$s$ predicted $i$ steps ahead of the current time step $t$. The set $\hat{\mathcal{O}}_{i|t}^s$ can be computed as\cite{gao2022risk}
\begin{subequations}\label{Eq:appreachset}
\begin{align}
\hat{\mathcal{O}}_{i|t}^s & = {\rm Proj}_{\rm position}(\hat{\mathcal{R}}_{i|t}^s), \label{eq_predicted_occupancy} \\
\hat{\mathcal{R}}_{i+1|t}^s  & = A_{t+i}^s \hat{\mathcal{R}}^s_{i|t} \oplus B_{t+i}^s \hat{\mathbb{U}}_t^s, \label{eq_predicted_reachable_set} \\
\hat{\mathcal{R}}_{0|t}^s & = \left\{ x_t^s \right\}, \label{eq_initial_reachable_set}  
\end{align}
\end{subequations}
where $\hat{\mathcal{R}}_{i|t}^s$ is the forward reachable set at time step $t+i$ based on the set $\hat{\mathbb{U}}_t^s$, $x_t^s$ is the measured state of the obstacle at time step $t$, ${\rm Proj}_{\rm position}(\hat{\mathcal{R}}_{i|t}^s)$ denotes the set of positions projected from the reachable set.

\begin{figure}[!t]
\centering
\subfloat[ ]{\includegraphics[width=0.45\columnwidth]{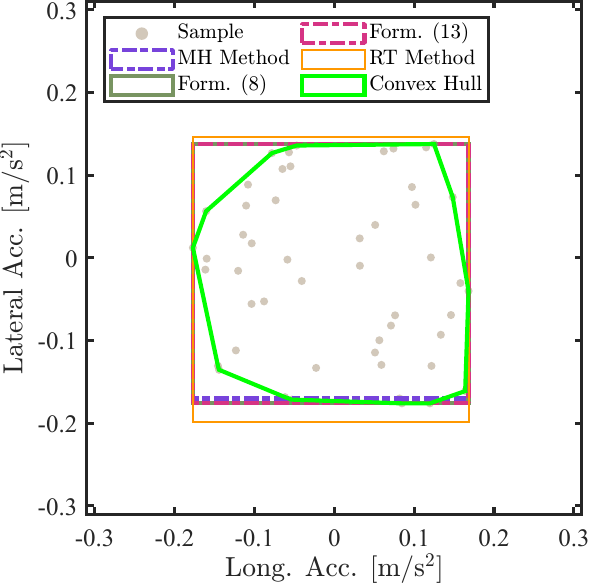}}
\hfil
\subfloat[ ]{\includegraphics[width=0.45\columnwidth]{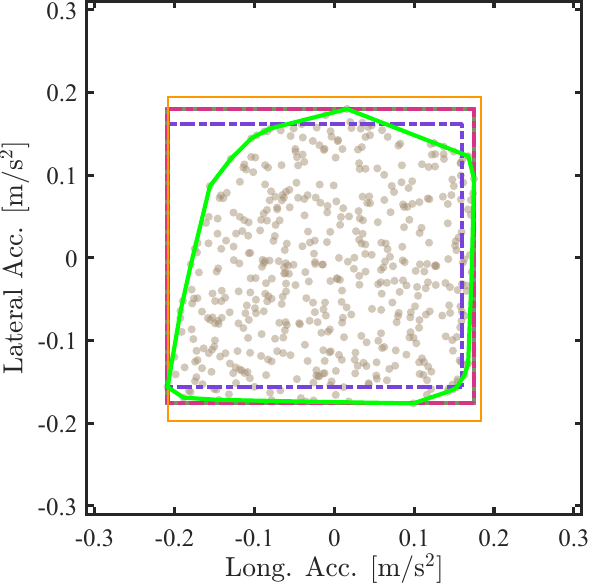}}\\
\subfloat[ ]{\includegraphics[width=0.45\columnwidth]{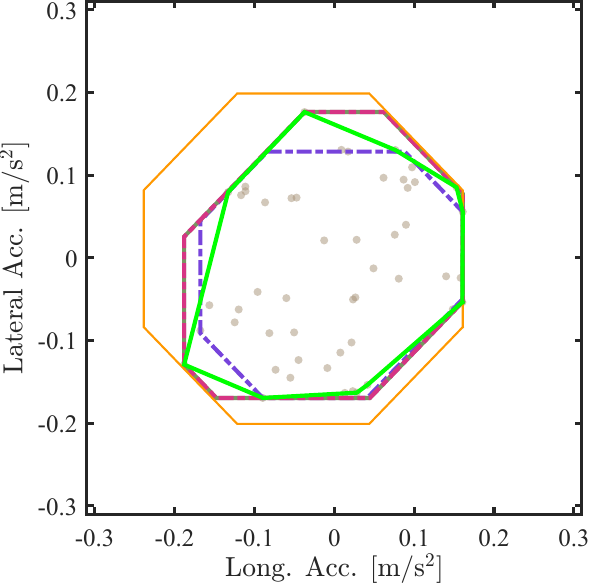}}
\hfil
\subfloat[ ]{\includegraphics[width=0.45\columnwidth]{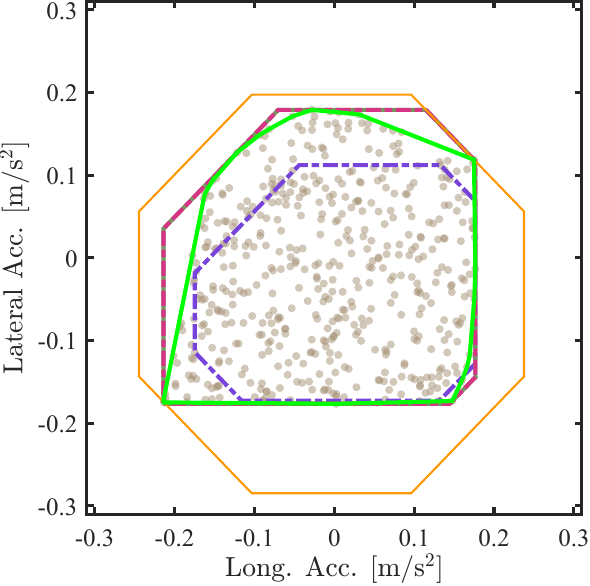}}\\
\subfloat[ ]{\includegraphics[width=0.9\columnwidth]{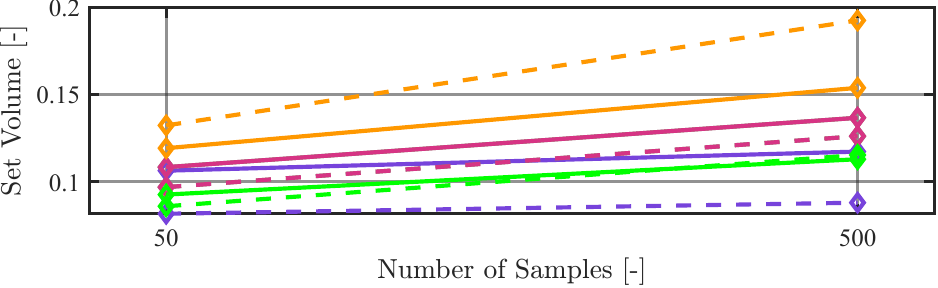}}
\centering
\caption{The learned intended control sets by formulation~\eqref{LP:quanset}, formulation~\eqref{LP:quansetonline}, the moving-horizon (MH) method, and the rigid tube (RT) method (for the MH method $L = 30$ ). (a) The number of samples is $50$, $\mathbb{U}^s$ is a square. (b) The number of samples is $500$, $\mathbb{U}^s$ is a square. (c) The number of samples is $50$, $\mathbb{U}^s$ is a regular hexagon. (d) The number of samples is $500$, $\mathbb{U}^s$ is a regular hexagon. (e) The set volume with colors according to Fig.~\ref{fig: set with different samples}(a). The dashed lines indicate the results with $\mathbb{U}^s$ as a regular hexagon, and solid lines indicate the results with $\mathbb{U}^s$ as a square.}
\label{fig: set with different samples}
\end{figure}

\begin{table*}[!t]
\scriptsize
\centering  
\caption{Computation Time of Different Methods for Learning the Set}  
\label{Table_set_est_com_time}  
\setlength{\tabcolsep}{1mm}
\begin{tabular}{lccccccccccc}  
\toprule  
\multicolumn{1}{c}{\textbf{Methods}} & \multicolumn{2}{c}{\textbf{Formulation \eqref{LP:quanset}}} & \multicolumn{2}{c}{\textbf{Formulation \eqref{LP:quansetonline}}} & \multicolumn{2}{c}{\textbf{Moving Horizon}} & \multicolumn{2}{c}{\textbf{Rigid Tube}} & \multicolumn{3}{c}{\textbf{Convex Hull}} \\
\cmidrule(r){2-3} \cmidrule(lr){4-5} \cmidrule(lr){6-7} \cmidrule(lr){8-9} \cmidrule(l){10-12}
& \multicolumn{1}{c}{Learn Set} & \multicolumn{1}{c}{Reach. Pred.} & \multicolumn{1}{c}{Learn Set} & \multicolumn{1}{c}{Reach. Pred.} & \multicolumn{1}{c}{Learn Set} & \multicolumn{1}{c}{Reach. Pred.} & \multicolumn{1}{c}{Learn Set} & \multicolumn{1}{c}{Reach. Pred.} & \multicolumn{1}{c}{Learn Set} & \multicolumn{1}{c}{Reach. Pred.} & \multicolumn{1}{c}{Vertices}  \\
\midrule
$50$ Samples, $\mathbb{U}^s$ as a square & 8.70 ${\rm ms}$ & 36.3 ${\rm ms}$ & 1.80 ${\rm ms}$ & 34.8 ${\rm ms}$ & 5.09 ${\rm ms}$& 32.4 ${\rm ms}$& 1.45 ${\rm ms}$& 32.0 ${\rm ms}$& 0.766 ${\rm ms}$& 1138 ${\rm ms}$ & 18  \\
$50$ Samples, $\mathbb{U}^s$ as a hexagon & 12.7 ${\rm ms}$ & 498 ${\rm ms}$ & 2.38 ${\rm ms}$ & 544 ${\rm ms}$ & 6.84 ${\rm ms}$& 500 ${\rm ms}$& 1.47 ${\rm ms}$& 531 ${\rm ms}$& 0.867 ${\rm ms}$ & 2477 ${\rm ms}$ & 22  \\\specialrule{0em}{1pt}{1pt}
$500$ Samples, $\mathbb{U}^s$ as a square & 40.1 ${\rm ms}$ & 38.4 ${\rm ms}$ & 1.61 ${\rm ms}$ & 35.8 ${\rm ms}$ & 4.92 ${\rm ms}$ & 34.8 ${\rm ms}$& 1.34 ${\rm ms}$& 34.2 ${\rm ms}$& 0.770 ${\rm ms}$& 23476 ${\rm ms}$ & 42 \\\specialrule{0em}{1pt}{1pt}
$500$ Samples, $\mathbb{U}^s$ as a hexagon & 101 ${\rm ms}$ & 514 ${\rm ms}$ & 2.19  ${\rm ms}$ & 560 ${\rm ms}$ & 6.97 ${\rm ms}$ & 568 ${\rm ms}$ & 1.35 ${\rm ms}$ & 555 ${\rm ms}$ & 0.737 ${\rm ms}$ & 12043 ${\rm ms}$ & 34 \\\specialrule{0em}{1pt}{1pt}
\bottomrule
\end{tabular}
\end{table*}

\subsection{Performance Evaluation}\label{sec: running examples}
This subsection evaluates different methods for learning the intended control set of an obstacle system, and the accuracy of the resultant forward reachability analysis.
\subsubsection{Comparison Between Different Methods for Learning the Intended Control Set}\label{sec: Comparison Between Different Methods for Learning the Intended Control Set}
The performance of the online recursion~\eqref{LP:quansetonline} and the moving-horizon (MH) method for learning the intended control set of the obstacle is compared with the original formulation \eqref{LP:quanset}, a rigid tube (RT) method \cite{gao2023robust}, and a method that directly learns the set as a convex hull of the samples in the information set $\mathcal{I}_t^s$ (the convex hull is computed using the Python package \texttt{pytope}~\cite{pytope}). Assume that the intended control set of the obstacle is learned by observing the longitudinal and lateral accelerations in the ground coordinate system. The results for learning the sets with different numbers of samples are shown in Fig.~\ref{fig: set with different samples}, where the admissible control set $\mathbb{U}^s$ is overly estimated as a square (4 vertices) and a regular hexagon (6 vertices), respectively. Furthermore, Table~\ref{Table_set_est_com_time} compares the computation time for learning the set and reachability prediction based on the learned set as in \eqref{Eq:appreachset}.

The analyses through Fig.~\ref{fig: set with different samples} and Table~\ref{Table_set_est_com_time} reflect that formulation~\eqref{LP:quansetonline} closely aligns with the original formulation \eqref{LP:quanset}, yet requires less computation time. Compared to the RT method, the learned set by \eqref{LP:quansetonline} is less conservative. The set obtained by the MH method is smaller than that by~\eqref{LP:quansetonline}. Notably, Fig.~\ref{fig: set with different samples} shows that the set expands when the number of samples increases, while Table~\ref{Table_set_est_com_time} shows that the computation time of \eqref{LP:quansetonline} and the MH method remains in the same order. The convex hull-based approach can more tightly contain the samples. However, it has a more complex geometry than the learned set obtained by other approaches, such that it renders a more complicated prediction of reachability. In addition, since the geometry of the convex hull is not deterministic, this is not beneficial for constructing the collision-avoidance constraint in the motion planner (see Example~\ref{exp: MPC formulation}). In comparison, other approaches can determine the geometry of the learned set by defining the shape of $\mathbb{U}^s$, such that the forward reachability prediction is more efficient. One additional observation from Table~\ref{Table_set_est_com_time} is that the shape of $\mathbb{U}^s$ can affect the computational efficiency of the reachability prediction. A more complex $\mathbb{U}^s$ has more degrees of freedom, while also increasing the computation time of reachability prediction. This implies a trade-off between accuracy and efficiency should be considered in practical applications when defining the admissible set $\mathbb{U}^s$.

\subsubsection{Responsiveness Analysis}\label{sec: response speed}
This subsection analyzes the responsiveness of the online recursion formulation~\eqref{LP:quansetonline} and the MH approach in learning the control set of an obstacle as in Section~\ref{sec: Comparison Between Different Methods for Learning the Intended Control Set}. In the setup, the obstacle performs mild actions before time step $t=50$, and from time step $t=51$ it shifts to executing more aggressive maneuvers. The volume of the learned set with the discrete time steps, as well as the 2-norm of the sampled control action ($||u_t^s||$), are presented in Fig.~\ref{fig: set volume with change action}. In addition, the learned control sets at different time steps are shown in Fig.~\ref{fig: set with mild and aggressive actions}. Note that the index of the time step in Fig.~\ref{fig: set volume with change action} is equal to the number of online samples as the method obtains one sample at every time step.

It is seen in Fig.~\ref{fig: set volume with change action} that after the system starts to take aggressive control actions from $t=51$, the learned control sets by~\eqref{LP:quansetonline} and the MH approach can both rapidly respond to the change, the delay between the set volume and the change of the control style is just $1$ step. Fig.~\ref{fig: set with mild and aggressive actions} shows how the learned sets by two approaches start from a minor initial set and grows when more samples are obtained. The set expands notably from $t=50$ to $t=54$ as a few aggressive control actions are sampled. At the end of sampling $t=100$, the set obtained by \eqref{LP:quansetonline} contains both the mild and aggressive control actions, while the set obtained by the MH approach only contains the aggressive control actions. This shows that \eqref{LP:quansetonline} predicts the motion of the obstacle based on all observed control actions of the obstacle system, while the MH approach captures the time-varying intent of the obstacle by discarding past information.

\begin{figure}[!t]
\centering
\includegraphics[width=\columnwidth]{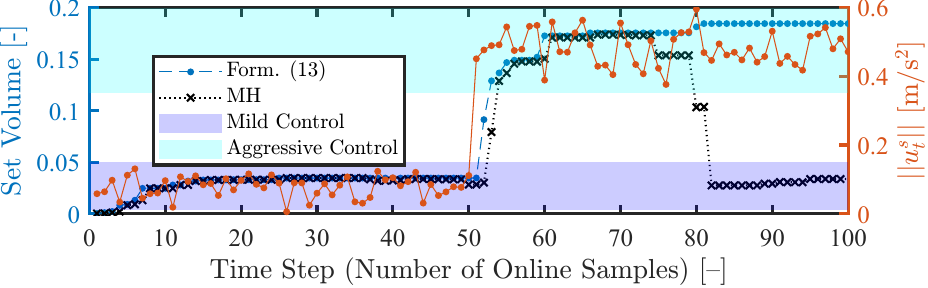}
\caption{Control actions of the system and volume of the learned set.}
\label{fig: set volume with change action}
\end{figure}
\begin{figure}[t]
\centering
\subfloat[ ]{\includegraphics[width=0.45\columnwidth]{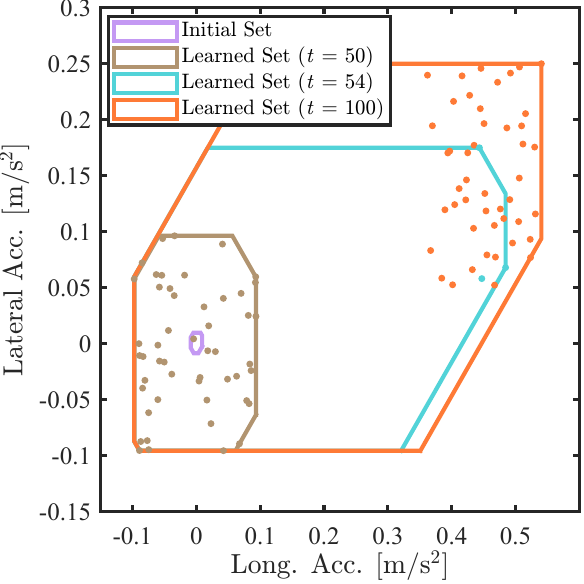}}
\subfloat[ ]{\includegraphics[width=0.45\columnwidth]{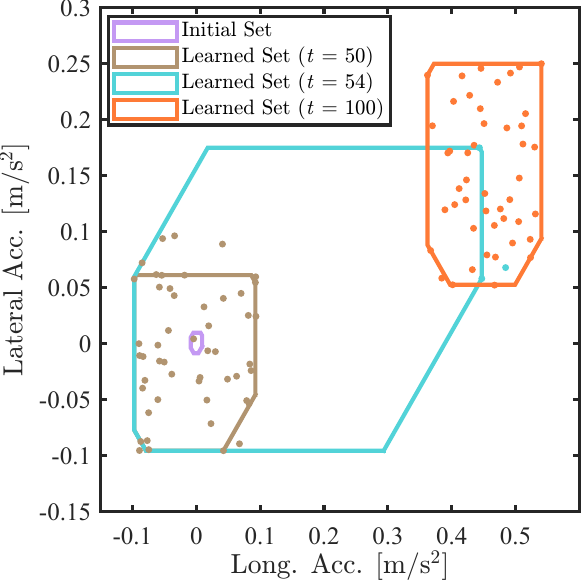}}
\centering
\caption{Learned control sets at different time steps (at $t=50$ the mild control behavior ends, and at $t=100$ the sampling ends). (a) With the online recursion formulation~\eqref{LP:quansetonline}. (b) With the MH approach.}
\label{fig: set with mild and aggressive actions}
\end{figure}
\begin{figure}[t]
\centering
\subfloat[ ]{\includegraphics[width=\columnwidth]{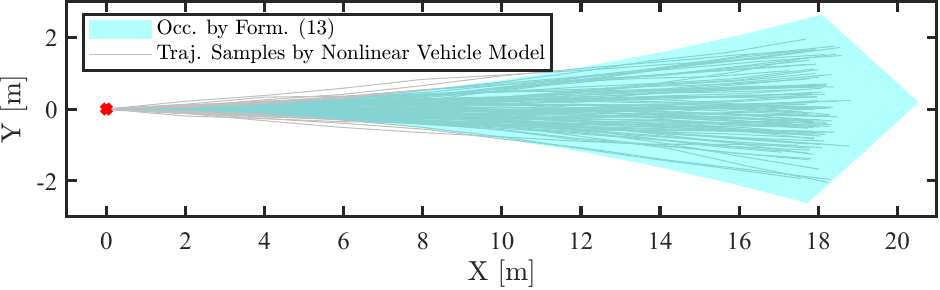}}\\
\subfloat[ ]{\includegraphics[width=\columnwidth]{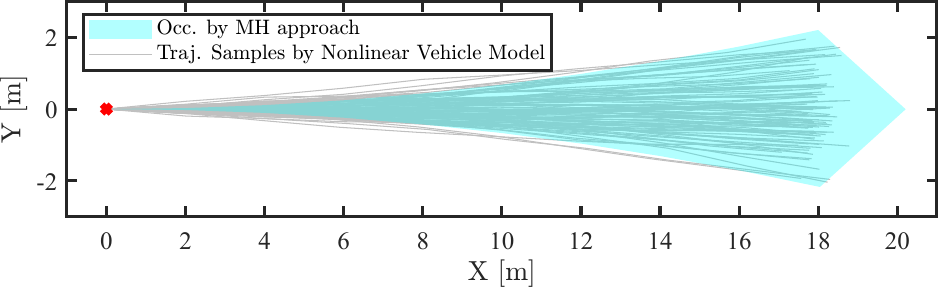}}
\centering
\caption{Predicted occupancy of the linear model~\eqref{eq_SV_model} by the proposed method based on the trajectory samples of a nonlinear vehicle model. (a) By the online recursion formulation~\eqref{LP:quansetonline}. (b) By the moving-horizon (MH) approach.}\label{fig: reach. set}
\end{figure}

\subsubsection{Reachability Analysis}\label{sec: reachability analysis}
This section evaluates the results of reachability analysis through \eqref{Eq:appreachset} based on the learned control set $\hat{\mathbb{U}}_t^s$. The actual obstacle dynamics are in this case modeled by a nonlinear kinematic single-track vehicle model as in \eqref{eq: EV model}, where the control inputs are generated from an unknown input set with an unknown distribution. The accelerations in the global coordinate system are extracted through historical trajectories to learn its intended control set by \eqref{LP:quansetonline} and the MH approach, respectively. Then it uses \eqref{Eq:appreachset} to predict the occupancy over a horizon, where the dynamics of the system are modeled by a second-order integrator as in~\eqref{eq_SV_model}. It is seen in Fig.~\ref{fig: reach. set} that the predicted occupancy using a linear model can well cover the trajectory samples generated by the nonlinear kinematic model, which indicates that the linear model-based prediction by the proposed method is a good local approximation of the nonlinear vehicle model with the choosen prediction horizon. Since the occupancy set in Fig.~\ref{fig: reach. set} can not completely cover the trajectory samples, it implies that the proposed method, in theory, does not guarantee safety in motion planning, but instead recovers approximately safe behavior by relying on past data to predict future obstacle dynamics. However, the empirical results in Section~\ref{sec: simulations} indicate that the proposed method still exhibits a high safety performance as a result of implementing the motion planner in a receding horizon manner. Note that the range of the occupancy in Fig.~\ref{fig: reach. set} depends on the maneuvers of the observed obstacle. If the control actions of the obstacle are generated from a larger input set, the predicted occupancy set will expand accordingly. To conclude, even though the nonlinear model of the obstacle and the linear second-order integrator model used for predicting the occupancy in \eqref{Eq:appreachset} are quite different, the proposed method still offers the desired accuracy in predicting the occupancy.

In addition, compared with a similar method, our proposed method has advantages in calculation time. Among the related work in Section~\ref{sec: uncertainty prediction}, the reachability analysis based on solving mixed-integer linear programming (MILP) problem \cite{driggs2018robust} exhibits the most similar characteristics to the proposed method as it does not require the distribution of behaviors and model information of the obstacle. The MILP-based method can get similar occupancy as in Fig.~\ref{fig: reach. set} by tuning the confidence level \cite[Fig. 4]{driggs2018robust}. For the same horizon, the average computation time by the proposed method is $0.0340 \ {\rm s}$ for formulation~\eqref{LP:quansetonline} and $0.0432 \ {\rm s}$ for the MH approach, with $100$ trajectories, and $0.0361 \ {\rm s}$ for formulation~\eqref{LP:quansetonline} and $0.0451 \ {\rm s}$ for the MH approach, with $500$ trajectories. For the MILP-based method, as presented in \cite[Table I]{driggs2018robust}, the average computation time  with $100$ trajectories and $500$ trajectories is $0.0322 \ {\rm s}$ and $0.1352 \ {\rm s}$. This reflects that the proposed method is more computationally efficient compared with the MILP-based approach, and the computation time remains unaffected by the number of sampled trajectories.
\begin{remark}\label{remark:persistence of excitation}
Assumption~\ref{assum:measure control} ensures the existence of a unique solution of the control input $u_{t-1} ^s$ in \eqref{eq_obs_LTV_model} to form a non-empty information set $\mathcal{I}_t^s$. Given a non-empty set $\mathcal{I}_t^s$, the convex optimization problem for determining the control set $\hat{\mathbb{U}}_t^s$, either by \eqref{LP:quanset} or \eqref{LP:quansetonline}, is feasible since the information set $\mathcal{I}_t^s$ is always contained within the admissible set $\mathbb{U}^s$. In addition, state-measurement uncertainties inherently contribute to uncertainties of $u_{t-1}^s$. The set $\hat{\mathbb{U}}_t^s$ learned with small uncertainties in $u_{t-1}^s$ would still be expected to center around the set learned with the true $u_{t-1}^s$ values. When these uncertainties are prominent, improving the accuracy of $\hat{\mathbb{U}}_t^s$ would depend on a more precise tracking system.
\end{remark}

\begin{remark}\label{remark:constant velocity}
For the double-integrator model~\eqref{eq_SV_model}, if the obstacle is observed with a constant velocity, then the estimated control input $u_{t-1}^s$ is always $\bm{0}$. In this case, the information set is a singleton, and the proposed motion-prediction method corresponds to a constant velocity model. The proposed method is designed to capture what has been observed from the obstacle.  In particular, if the obstacle only moves with a constant velocity, the method is supposed to reflect that behavior. However, as soon as the obstacle changes the driving behavior, the method can adapt to this change with one time step, as indicated in Figs.~\ref{fig: set volume with change action}--\ref{fig: set with mild and aggressive actions}. The information set transitions from a singleton to a non-singleton when a non-zero control input from the obstacle is estimated, thereby improving the robustness of the prediction over the constant velocity model.
\end{remark}

\section{Robust Motion-Planning Strategy}\label{General Robust MPC for Motion Planning}

This section formulates the robust MPC for safe motion planning of the ego system, based on the predicted obstacle occupancy $\hat{\mathcal{O}}_{i|t}^s$ defined in \eqref{eq_predicted_occupancy}. During the online motion-planning process, at each time step $t$, the ego system updates the set $\hat{\mathbb{U}}_t^s$ and the predicted occupancy $\hat{\mathcal{O}}_{i|t}^s$, and then solves the following  optimal control problem (OCP)
\begin{subequations}\label{eq: general ocp}
\begin{align}
\mathop{\rm minimize}\limits_{u^e_{i-1|t}} \quad & {V}\left(x^e_{i|t}, \ u^e_{i-1|t}, \ {r} \right) & \label{eq_ocp_a} \\
{\rm subject \ to}\quad
&x^e_{i|t} = f(x^e_{i-1|t}, u^e_{i-1|t}), & \label{eq_ocp_b} \\
&x^e_{i|t} \in \mathbb{X}^e,  &\label{eq_ocp_c}\\
&u_{i-1|t}^e \in \mathbb{U}^e, &\label{eq_ocp_d}\\
& {\rm dis}(p^e_{i|t}, \hat{\mathcal{O}}_{i|t}^s)>d^s_{\min}, \ \forall s \in \mathcal{S}, & \label{eq_ocp_e} 
\end{align}
\end{subequations}
\noindent where $i = 1, \ldots, N$. In the cost function \eqref{eq_ocp_a}, ${r}$ denotes the reference state.  The cost function ${V}(\cdot)$ is defined for measuring the performance of the motion planning. The constraints \eqref{eq_ocp_b}--\eqref{eq_ocp_d} concerning the ego system have been introduced in \eqref{eq_ego_model}. The constraint \eqref{eq_ocp_e} refers to the safety requirement on the distance between the ego system and obstacle $s$, where ${\rm dis}(\cdot)$ is a distance measure between $p^e_{i|t}$ the position of center of geometry of the ego system and the obstacle's occupancy $\hat{\mathcal{O}}_{i|t}^s$ \cite{zhang2021optimization}. It means at time step $t+i$ in the prediction horizon, the distance between $p^e_{i|t}$ and $\hat{\mathcal{O}}_{i|t}^s$ should be larger than the minimum safety distance $d^s_{\rm min}$. Note that $p_{i|t}^e$ and  $\hat{\mathcal{O}}_{i|t}^s$ predict the motion of the center of geometry of the ego and obstacle systems, such that the actual shape factors of them should be considered in designing $d^s_{\rm min}$. In addition, a large occupancy $\hat{\mathcal{O}}_{i|t}^s$ may cause infeasibility of the motion-planning problem as a result of the limited driving area. This can be handled by introducing a slack variable to \eqref{eq_ocp_e} to enhance the feasibility, and this will be enforced in Example~\ref{exp: MPC formulation} with the explicit and equivalent expressions of constraint~\eqref{eq_ocp_e}.

The motion planning of the ego system is performed by solving the OCP \eqref{eq: general ocp} at every time step $t$ to obtain the optimal control sequence $\{u^{e,\star}_{i-1|t}\}_{i=1}^{N}$, which is applied to steer the model \eqref{eq_ocp_b} to generate the reference trajectory for a lower-level controller at the current time step.

\begin{example}\label{exp: MPC formulation}
{\textbf{MPC formulation of the reach-avoid problem:}} Consider the models of EV and SV in Example~\ref{Ex:examplemodel}. Based on the occupancy $\hat{\mathcal{O}}_{i|t}^s$ of the SV, the EV model \eqref{eq: EV model}, and formulation~\eqref{eq: general ocp}, the OCP in this case is designed as 
\begin{subequations}\label{eq: specific ocp for motion planning}
\begin{align}
\mathop{\rm minimize}\limits_{u_{i-1|t}^e, \ \lambda_{i}^s, \ \varepsilon_i^s} \quad & \sum_{i=1}^{N}\left(||\delta_{i-1|t}^e||_{Q_1}^2 +
|| \eta_{i-1|t}^e||_{Q_2}^2\right) + \notag \\ & 
||E_{N|t}||_{Q_3}^2 + \sum_{i=1}^{N}\left(||\varepsilon_i^s||_{Q_4}^2\right)& \label{eq_speot_ocp_ea_a} \\
{\rm subject \ to}\quad
&x_{i|t}^e = f(x_{i-1|t}^e, \ u_{i-1|t}^e), & \label{eq_speot_ocp_ea_b}\\
&\underline{\mathcal{U}} \leq \left[v_{i|t}^e \ a_{i|t}^e \ \delta_{i-1|t}^e\right]^{\top} \leq \overline{\mathcal{U}}, &\label{eq_speot_ocp_ea_c}\\
& {p_{i|t}^e} \in \mathcal{D},& \label{eq_speot_ocp_ea_d} \\
&(H_{i|t}^sp_{i|t}^e - h_{i|t}^s)^{\top}\lambda_i^s  \geq d^s_{\rm min} - \varepsilon_i^s, &\label{eq_speot_ocp_ea_e}\\
&||(H_{i|t}^s)^{\top}\lambda_i^s||_2  = 1, \label{eq_speot_ocp_ea_f} \\
&\lambda_i^s  \in \mathbb{R}_{+}^{n_h}, \ 0\leq \varepsilon_i^s \leq d^s_{\rm min},\label{eq_speot_ocp_ea_g}
\end{align}
\end{subequations}
\noindent where $Q_1$,$\cdots$,$Q_4$ are weighting matrices, \eqref{eq_speot_ocp_ea_b} is the discrete approximation of the continuous model~\eqref{eq: EV model} using a fourth-order Runge-Kutta method with the sampling interval $T$, $u_{i-1|t}^e=[\delta_{i-1|t}^e \ \eta_{i-1|t}^e]$ is the control input of model \eqref{eq: EV model}, and $x_{i-1|t}^e$ is the state of model \eqref{eq: EV model}. Parameters $\underline{\mathcal{U}}$ and $\overline{\mathcal{U}}$ mean the lower and upper bounds on the velocity, acceleration, and front tire angle of the EV. Constraint \eqref{eq_speot_ocp_ea_d} limits the position of the center of geometry of the EV, $p_{i|t}^e = [p_{x, i|t}^e \ p_{y, i|t}^e]^{\top}$, within the driveable area $\mathcal{D}$. Constraints \eqref{eq_speot_ocp_ea_e}--\eqref{eq_speot_ocp_ea_g} are an equivalent reformulation of the collision-avoidance constraint \eqref{eq_ocp_e} with additional decision variable $\lambda_i^s$ and slack variable $\varepsilon_i^s$ \cite[Section III-B]{zhang2021optimization}. In practical applications, the variable $\varepsilon_i^s$ should be penalized with a high weight to make sure that the hard collision-avoidance constraint is violated only in extreme cases. The parameter $n_h$ is the number of edges of the set $\hat{\mathcal{O}}^s_{i|t}$, and the matrix $H^s_{i|t}$ and the vector $h^s_{i|t}$ are given by
\begin{equation*}
\hat{\mathcal{O}}^s_{i|t} = \left\{x \in \mathbb{R}^{n_p}: H^s_{i|t}x \leq h^s_{i|t} \right\}. \label{eq_O_poly}  
\end{equation*}

Define $l^e$ and $w^e$ as the length and width of the EV, $l^s$ and $w^s$ as the length and width of the SV. The safety distance $d^s_{\rm min}$ in \eqref{eq_speot_ocp_ea_e} is designed as
\begin{equation}
d^s_{\rm min} = \sqrt{(l^e/2)^2 + (w^e/2)^2} + \sqrt{(l^s/2)^2 + (w^s/2)^2}. \label{eq_circle}
\end{equation}

Given $v^{\rm f}$, $x^{\rm f}$, $y^{\rm f}$, and $\varphi^{\rm f}$ as the predefined reference velocity, longitudinal and lateral positions, and heading angle of the EV, the terminal error vector $E_{N|t}$ in \eqref{eq_speot_ocp_ea_a} is defined as 
\begin{equation}
E_{N|t}=[v^e_{N|t} - v^{\rm f} \ x^e_{N|t} - x^{\rm f} \ y^e_{N|t} - y^{\rm f} \ \varphi^e_{N|t} - \varphi^{\rm f}]. \label{eq_terminal_vec_err_EA}
\end{equation}

The formulation~\eqref{eq_circle} is designed to achieve a trade-off between the geometric accuracy in the description of the SV and the computational efficiency of the method. For higher geometric accuracy, it is necessary to predict the reachability with consideration of the orientation of the SV. However, integrating the orientation into the obstacle occupancy will make the obstacle occupancy significantly more complex. This also implies that more decision variables or constraints are needed to formulate the collision-avoidance constraints in the motion planner (see \cite[Section IV]{zhang2021optimization} and \cite[Section IV.A]{fan2024efficient}). To this end, \eqref{eq_circle} is applied to approximate the geometry of the SV. Although this approach may introduce possible conservatism, the potential feasibility issues are addressed in practical applications by using soft constraints as in \eqref{eq_speot_ocp_ea_e}.
\end{example}

\section{Simulations of Reach-Avoid Motion Planning}\label{sec: simulations}
\begin{figure}[t]
\centering
\includegraphics[width=0.8\columnwidth]{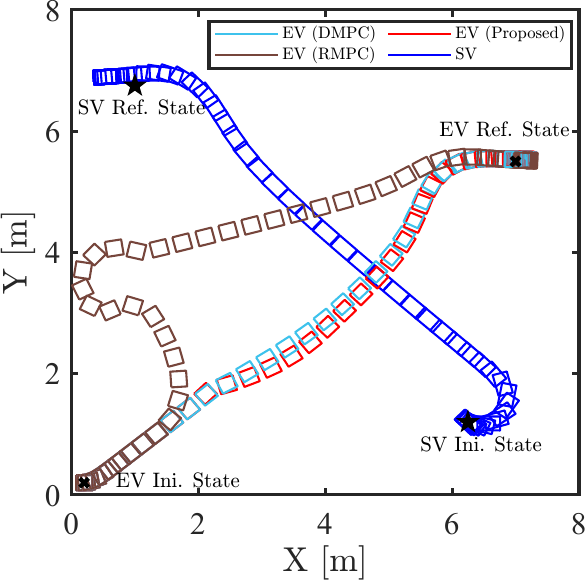}
\caption{Comparison of the planned paths of EV by different approaches in the reach-avoid scenario. The initial state of the EV is $(0.2 \ {\rm m}, \ 0.2 \ {\rm m}, \ 0 \ {\rm rad}, \ 0 \ {\rm m/s}, \ 0 \ {\rm m/s^2})$, the reference is $(7 \ {\rm m}, \ 5.5 \ {\rm m}, \ 0 \ {\rm rad}, \ 0 \ {\rm m/s}$). The initial state of the SV is $(6.25 \ {\rm m},\ 1.2 \ {\rm m}, \ -1/4\pi \ {\rm rad}, \ 0 \ {\rm m/s})$, and the reference state is $(1 \ {\rm m}, \ 6.75 \ {\rm m}, \ \pi \ {\rm rad}, \ 0 \ {\rm m/s})$.}
\label{fig:SIM4_path_comparison_two_vehicles}
\end{figure}

\begin{table}[!t]
\centering
\caption{General Parameters in Simulations}  
\label{Table_exp_parameters robotic} 
\begin{tabular}{cccc}
\toprule
\textbf{Symbol} & \textbf{Value} & \textbf{Symbol} & \textbf{Value} \\
\midrule
$l_f$, $l_r$ & $0.08 \ {\rm m}$, $0.08 \ {\rm m}$ & $T$& $0.25 \ {\rm s}$ \\ \specialrule{0em}{1pt}{1pt}
$l^e$, $w^e$ & $0.26 \ {\rm m}$, $0.25 \ {\rm m}$ & $l^s$, $w^s$ & $0.36 \ {\rm m}$, $0.23 \ {\rm m}$ \\ \specialrule{0em}{1pt}{1pt}
$N$ & 10 & $Q_1$, $Q2$ & $1$, $1$ \\ \specialrule{0em}{1pt}{1pt}
$Q_3$ &  $\rm{diag}([1\ 5 \ 5 \ 2])$ & $Q_4$ &  $300$ \\ \specialrule{0em}{1pt}{1pt}
$\underline{\mathcal{U}}$ & $[-1.5 \ -0.5 \ -0.3]^{\top}$ & $\overline{\mathcal{U}}$& $[1.5  \ 0.5 \ 0.3]^{\top}$\\ \specialrule{0em}{1pt}{1pt}
\bottomrule
\end{tabular}
\begin{tablenotes}\tiny
\item ${\star}$ The units in $\overline{\mathcal{U}}$ and $\underline{\mathcal{U}}$ are ${\rm m/s}$, ${\rm m/s^2}$, and ${\rm rad}$, respectively.
\end{tablenotes} 
\end{table}

This section presents simulation results for reach-avoid planning as shown in Fig.~\ref{fig:Motivation_Example}, where the EV and SV are specified as a car-like mobile robot in this case. The implementations can be found in our published code\footnote{{\tiny{https://github.com/JianZhou1212/robust-mpc-motion-planning-by-learning-obstacle-uncertainties}}}. The mission of the EV is to follow a predefined reference state while avoiding collision with the SV and borders of the driveable area. It is assumed that the SV holds a higher priority, requiring the EV to actively take measures to avoid collisions. The controller of the SV is simulated by a nonlinear MPC to track its reference target, where the reference target and the model and controller parameters are unknown to the EV. The SV takes control input from an unknown set to plan a reference trajectory to reach its reference target. The SV's dynamics are not available to the EV, from the perspective of the EV, the SV is modeled by the linear model~\eqref{eq_SV_model}.

In the numerical studies, the computation of the occupancy of the SV firstly requires formulating an admissible control set $\mathbb{U}^s$ according to Assumption~\ref{assum:poly_U}. Since the simulation runs for just $55$ time steps, the size of the information set $\mathcal{I}_t^s$ is not overly big. Therefore, the online recursion~\eqref{LP:quansetonline} is more suitable for learning the unknown intended control set of SV in this case. This is initialized by an initial information set $\mathcal{I}_0^s$ that contains some artificial samples to construct a small but non-empty set $\hat{\mathbb{U}}_0^s$ by solving \eqref{LP:quanset}. This implies that the proposed method does not have any prior information on the uncertainties of the SV. Then, for $t \in \mathbb{N}_{+}$ the set $\hat{\mathbb{U}}_t^s$ is obtained by solving \eqref{LP:quansetonline} based on $\hat{\mathbb{U}}_{t-1}^s$ and $u_{t-1}^s$, where $\hat{\mathbb{U}}_{t-1}^s$ is obtained at time step $t-1$, and $u_{t-1}^s$ contains the measured longitudinal and lateral accelerations of the SV in the ground coordinate system at time step $t-1$. The learned set $\hat{\mathbb{U}}_t^s$ is substituted into \eqref{Eq:appreachset} to obtain $\hat{\mathcal{O}}_{i|t}^s$, i.e., the predicted occupancy of the SV, for planning of the EV by solving \eqref{eq: specific ocp for motion planning}.

The proposed planner is compared with a robust MPC (RMPC) and a deterministic MPC (DMPC), where the RMPC and DMPC are implemented by replacing $\hat{\mathbb{U}}_t^s$ in \eqref{eq_predicted_reachable_set} with $\mathbb{U}^s$ and ${\bm 0}$, respectively. For the DMPC, the obstacle system~\eqref{eq_SV_model} is predicted with ${\bm 0}$ input, such that the prediction is equivalent to a constant velocity model. The set addition and polytope computations were implemented by the Python package \texttt{pytope} \cite{pytope}. The optimization problems involved were solved by \texttt{CasADi} \cite{andersson2019casadi} and \texttt{Ipopt} \cite{wachter2006implementation} using the linear solver \texttt{MA57} \cite{hsl2021collection}. The simulations were performed on a standard laptop running Ubuntu 22.04 LTS and Python 3.10.12. Table~\ref{Table_exp_parameters robotic} provides a summary of simulation parameters. 

\begin{figure}[!t]
\centering
\subfloat[ ]{\includegraphics[width=\columnwidth]{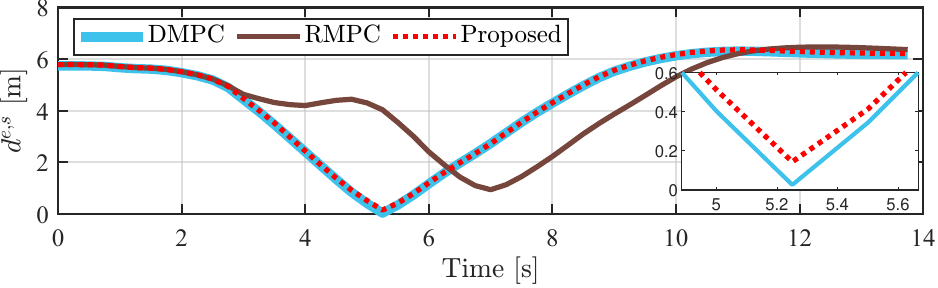}}
\vfil
\subfloat[ ]{\includegraphics[width=\columnwidth]{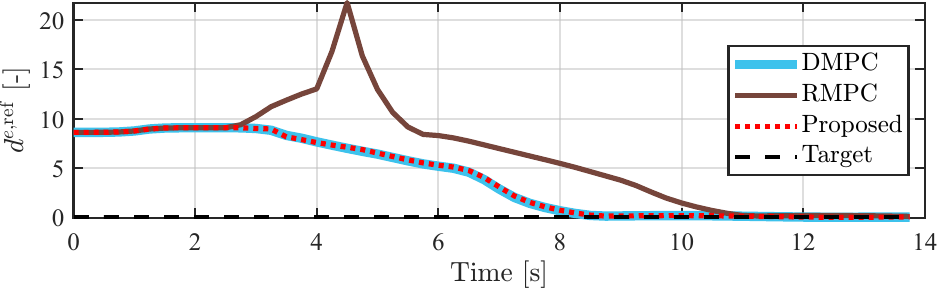}}
\vfil
\subfloat[ ]{\includegraphics[width=\columnwidth]{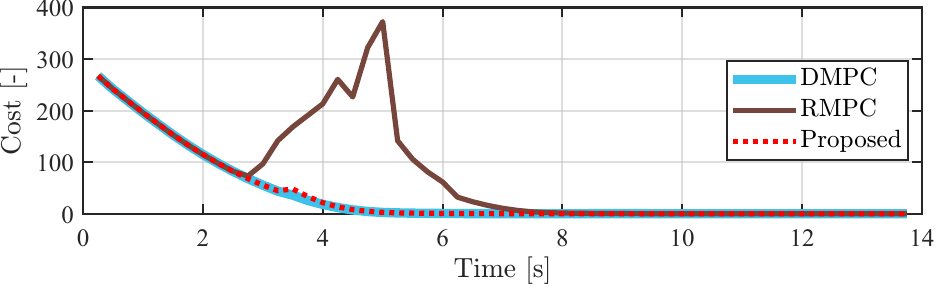}}
\centering
\caption{Comparison of $d^{e, s}$, $d^{e, \rm ref}$, and the cost during the motion. (a) The polytope distance between EV and SV ($d^{e, s}$) during the motion (the minimum EV-SV distance with DMPC is $0.03 \ {\rm m}$, and that with the proposed method is $0.14 \ {\rm m}$). (b). The Euclidean distance between EV and its reference target ($d^{e, \rm ref}$) during the motion. (c). The cost function value.}
\label{fig:SIM4_distance_comparison_two_vehicle}
\end{figure}

\begin{table*}[!t]
\tiny
\centering  
\caption{Comparison of the Performance of Three Planners in Stochastic Simulations}  
\label{Table_stochastic_comparison_robot}  
\setlength{\tabcolsep}{1mm}
\begin{threeparttable}
\begin{tabular}{ccccccccccccccccc}  
\toprule  
\multicolumn{1}{c}{Name} & \multicolumn{2}{c}{Collision-free rate} & \multicolumn{2}{c}{Complete rate} & \multicolumn{2}{c}{Mean $d_{\rm min}^{e,s}$} & \multicolumn{2}{c}{Min. $d_{\rm min}^{e,s}$} & \multicolumn{2}{c}{Mean $\tau^{e, \rm ref}$} & \multicolumn{2}{c}{Max. $\tau^{e, \rm ref}$} & \multicolumn{2}{c}{Mean $J_{\rm sum}$}& \multicolumn{2}{c}{Max.  $J_{\rm sum}$} \\
\cmidrule(lr){2-3} \cmidrule(lr){4-5} \cmidrule(lr){6-7} \cmidrule(lr){8-9} \cmidrule(lr){10-11} \cmidrule(lr){12-13} \cmidrule(lr){14-15} \cmidrule(lr){16-17} 
& \multicolumn{1}{c}{\(N=10\)} & \multicolumn{1}{c}{\(N=8\)} & \multicolumn{1}{c}{\(N=10\)} & \multicolumn{1}{c}{\(N=8\)} & \multicolumn{1}{c}{\(N=10\)} & \multicolumn{1}{c}{\(N=8\)} & \multicolumn{1}{c}{\(N=10\)} & \multicolumn{1}{c}{\(N=8\)} & \multicolumn{1}{c}{\(N=10\)} & \multicolumn{1}{c}{\(N=8\)} & \multicolumn{1}{c}{\(N=10\)} & \multicolumn{1}{c}{\(N=8\)} & \multicolumn{1}{c}{\(N=10\)} & \multicolumn{1}{c}{\(N=8\)} & \multicolumn{1}{c}{\(N=10\)} & \multicolumn{1}{c}{\(N=8\)} \\
\midrule
Proposed & $100\%$ & $100\%$ & $100\%$ & $100\%$ & $0.105 \ {\rm m}$ & $0.081\ {\rm m}$ & $0.057\ {\rm m}$ & $0.039\ {\rm m}$ & $8.75 \ {\rm s}$ & $10.5 \ {\rm s}$  & $8.75 \ {\rm s}$ & $10.5 \ {\rm s}$ & $2001$ & $2576$ & $2301$ & $2837$ \\ 
RMPC & $98.3\%$ & $100\%$ & $65.1\%$ & $80\%$ & $0.689 \ {\rm m}$ & $0.184\ {\rm m}$ & $0.132\ {\rm m}$ & $0.083\ {\rm m}$ & $11.3 \ {\rm s}$ & $11.8 \ {\rm s}$  & $13.75\ {\rm s}$ & $13.25 \ {\rm s}$ & $4496$ & $3452$ & $5961$ & $4702$ \\ 
DMPC & $40.3\%$ & $30.7\%$ & $100\%$ & $100\%$ & $0.020\ {\rm m}$ & $0.019\ {\rm m}$ & $0.010\ {\rm m}$ & $0.010\ {\rm m}$ & $8.75\ {\rm s}$ & $10.5\ {\rm s}$ & $8.75\ {\rm s}$ & $10.5\ {\rm s}$& $1967$ & $2553$  & $1983$ & $2680$\\
\bottomrule
\end{tabular}
\begin{tablenotes}
\item ${\star}$ The complete rate, $d_{\rm min}^{e, s}$, $\tau^{e, \rm ref}$, and $J_{\rm sum}$ are counted among the collision-free cases.
\item ${\star}$ A collision is counted if the distance between the EV and SV is less than or equal to $0.01 \ {\rm m}$ during the simulation.
\end{tablenotes}
\end{threeparttable}
\end{table*}

\begin{table*}[!t]
\scriptsize
\centering  
\caption{Comparison of Computational Performance of Three Methods}  
\label{Table_methods_com_time} 
\setlength{\tabcolsep}{1mm}
\begin{tabular}{ccccccccc}
\toprule  
\textbf{Methods} & \multicolumn{2}{c}{\textbf{Learn Set \eqref{LP:quansetonline}}} & \multicolumn{2}{c}{\textbf{Reach. Prediction} \eqref{Eq:appreachset}} & \multicolumn{2}{c}{\textbf{Optimization \eqref{eq: specific ocp for motion planning}}}  & \multicolumn{2}{c}{\textbf{Total}}  \\
\cmidrule(r){2-3} \cmidrule(lr){4-5} \cmidrule(lr){6-7}  \cmidrule(lr){8-9} 
& \multicolumn{1}{c}{$N=10$} & \multicolumn{1}{c}{$N=8$} & \multicolumn{1}{c}{$N=10$} & \multicolumn{1}{c}{$N=8$} & \multicolumn{1}{c}{$N=10$} & \multicolumn{1}{c}{$N=8$}  & \multicolumn{1}{c}{$N=10$} & \multicolumn{1}{c}{$N=8$} \\
\midrule
Proposed & 1.24 $\pm$ 0.806 $\rm{ms}$& 1.23 $\pm$ 0.802 $\rm{ms}$ & 48.7 $\pm$ 2.19 $\rm{ms}$ & 28.0 $\pm$ 2.56 $\rm{ms}$ & 48.4 $\pm$ 15.6 $\rm{ms}$& 42.5 $\pm$ 14.0 $\rm{ms}$ & 98.5 $\pm$ 15.8 $\rm{ms}$ 
& 72.0 $\pm$ 14.5 $\rm{ms}$ \\\specialrule{0em}{1pt}{1pt}
DMPC & None & None & 5.87 $\pm$ 1.59 $\rm{ms}$ & 4.90 $\pm$ 1.63 $\rm{ms}$ & 50.7 $\pm$ 21.4 $\rm{ms}$& 44.9 $\pm$ 15.1 $\rm{ms}$ & 56.7 $\pm$ 21.5 $\rm{ms}$ 
& 49.8 $\pm$ 15.2 $\rm{ms}$ \\\specialrule{0em}{1pt}{1pt}
RMPC & None & None & 48.8 $\pm$ 2.53 $\rm{ms}$ & 27.9 $\pm$ 1.97 $\rm{ms}$ & 62.4 $\pm$ 29.1 $\rm{ms}$& 42.9 $\pm$ 22.2 $\rm{ms}$ & 111.3 $\pm$ 29.3 $\rm{ms}$ 
& 70.9 $\pm$ 22.3 $\rm{ms}$ \\
\bottomrule
\end{tabular}
\begin{tablenotes}\tiny
\item ${\star}$ The results are presented with Mean $\pm$ One standard deviation.
\end{tablenotes} 
\end{table*}

The methods are first implemented in a single case to get insights into the performance of each planner. The comparison is based on the planned paths, the polytope distance between the EV and SV ($d^{e, s}$), the Euclidean distance between the EV and its reference state ($d^{e, \rm ref}$), which reflects the convergence speed to the reference, and the cost function value. The results are shown in Fig.~\ref{fig:SIM4_path_comparison_two_vehicles} and Fig.~\ref{fig:SIM4_distance_comparison_two_vehicle}, respectively. It is observed in Fig.~\ref{fig:SIM4_path_comparison_two_vehicles} that the three methods perform differently in the same scenario as a result of different considerations of uncertainties. Further analysis from Fig.~\ref{fig:SIM4_distance_comparison_two_vehicle}(a) shows that the minimum distance between EV and SV with DMPC is close to $0 \ {\rm m}$, and that with the proposed method is $0.14 \ {\rm m}$, which is sufficient in this case considering that the length of the EV is $0.26 \ {\rm m}$. In contrast, RMPC can generate the safest reference trajectory as it maintains a larger distance between EV and SV. However,  the cost of pursuing robustness is that the EV needs a longer time and a larger cost to converge to the reference state, as shown in Figs.~\ref{fig:SIM4_distance_comparison_two_vehicle}(b)--(c), while the DMPC and the proposed method converge faster with a smaller cost. 

The three methods are further compared through $300$ Monte-Carlo simulations with a randomly sampled initial state for the SV, while other conditions are the same as in Fig.~\ref{fig:SIM4_path_comparison_two_vehicles}. The planners are implemented with the prediction horizons $N=10$ and $N=8$, respectively, and are compared by counting the collision-free rate of planning, where a collision-free mission is fulfilled if the EV has no collision with either the SV or the driving-area boundary. Among all successful cases, the complete cases are counted, where a complete mission means that the Euclidean distance between the state of the EV and its reference state converges to less than or equal to $0.2$ within a predefined time limitation ($13.75 \ {\rm s}$ in the simulation). For all completed cases with each planner, the minimum polytope distance between the EV and SV of each simulation ($d_{ \rm min}^{e, s}$), the first time instant when the EV reaches the reference state ($\tau^{e, \rm ref}$), and the sum of the cost function value ($J_{\rm sum}$) are compared. The results are presented in Table~\ref{Table_stochastic_comparison_robot}. It is seen that only the proposed method achieves a collision-free rate of $100\%$ for both prediction horizons. Note that the RMPC does not achieve $100\%$ with $N=10$ because the obstacle occupancy expands with the prediction horizon. This reduces the driveable space, i.e., the feasible region, for the EV in the limited area, thereby causing collisions with the boundary in some cases. Among all collision-free cases, the proposed method and DMPC have a $100\%$ complete rate, while the RMPC does not fulfill all the tasks. Analyzing the random variables $d_{ \rm min}^{e, s}$, $\tau^{e, \rm ref}$, and $J_{\rm sum}$ reveals that both the proposed method and DMPC guide the EV to the reference state at the same time instant, and the costs are close. However, the proposed method demonstrates superior safety compared to DMPC. On the other hand, RMPC ensures a large distance between the EV and SV in each simulation, but it takes more time and a larger cost to converge. For the proposed method, although the prediction horizons $N=10$ and $N=8$ both result in $100\%$ collision-free rate and complete rate, the shorter horizon can cause a smaller minimum distance between the EV and SV, and a larger cost of the maneuver. The simulation results demonstrate that the proposed method performs successfully in the absence of prior knowledge, e.g., the unknown dynamics, unknown distribution of control actions, and uncertain intentions of the obstacle. Empirically, the results show that the proposed approach is safe in all considered instances, while simultaneously enhancing efficiency and feasibility and reducing the cost associated with motion-planning tasks compared with the RMPC. In contrast, when compared to the DMPC, it markedly increases safety without incurring additional costs or compromising the efficiency of executing the tasks. 

The computation time for each iteration of the proposed method, DMPC, and RMPC with different prediction horizons is summarized in Table~\ref{Table_methods_com_time}. The results were recorded from $20$ random simulations where each MPC algorithm was run $55$ steps. It is seen in Table~\ref{Table_methods_com_time} that DMPC is the most computationally efficient, while the prediction horizon has an obvious influence on the computational efficiency of the proposed method and RMPC. In addition, for the proposed approach, the majority of the computation load comes from predicting the forward reachable set~\eqref{Eq:appreachset} and solving the MPC problem~\eqref{eq: specific ocp for motion planning}. The computation time for solving ~\eqref{eq: specific ocp for motion planning} increases slightly with a longer prediction horizon, while the polytope-based computation of occupancy sets using ~\eqref{Eq:appreachset} can
be computationally expensive when the prediction horizon is long. Table~\ref{Table_stochastic_comparison_robot} shows that the computation time is sufficient for running the planner at around ${\rm {10 \ Hz}}$ in simulations, while it further indicates that a trade-off between the performance and the computational efficiency should be carefully considered when defining the prediction horizon in practical applications.

\begin{remark}\label{remark: choose RMPC and DMPC}
The DMPC and RMPC are selected as the baseline methods for two primary reasons. First, in the studied reach-avoid problem, the DMPC excels in planning the most efficient reference trajectory when the motion-planning problem is feasible, while the RMPC is known for its guaranteed safety when the problem is feasible. The evaluations of the proposed method are demonstrated by comparing it with the efficiency-optimal method and the safety-optimal method, respectively. Second, similar to the proposed method, the DMPC and RMPC rely on a few assumptions on the traffic scenario. Other methods such as risk-aware optimal control \cite{gao2022risk}, interaction and safety-aware MPC \cite{zhou2022interactionnew}, and non-conservative stochastic MPC \cite{benciolini2023non}, \cite{brudigam2023stochastic}, require certain assumptions on uncertainty distributions of the SVs, which are not satisfied in our case studies. Therefore, DMPC and RMPC provide a fair comparison for assessing the proposed method's performance.
\end{remark}

\section{Hardware Experiments}\label{sec: hardware exp}
\begin{figure}[!t]
\centering
\includegraphics[height=0.4\columnwidth]{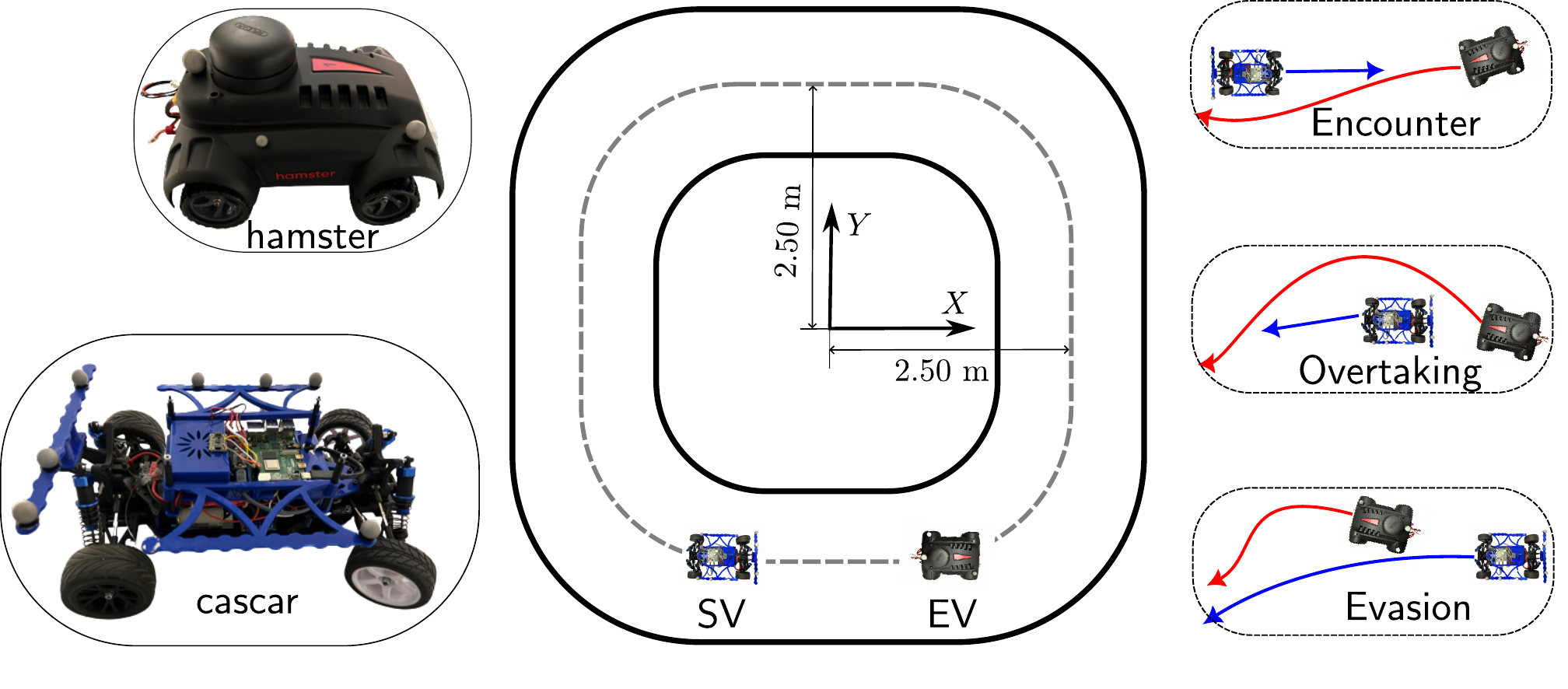}
\caption{The experiment scenarios of motion planning of an EV moving along a track with an SV.}
\label{fig:two_vehicle_interaction}
\end{figure}

\begin{table}[!t]\tiny
\centering
\caption{Experimental Parameters in Each Scenario}  
\label{Table_spe_exp_parameters robotic} 
\begin{tabular}{llcc}
\toprule
\multicolumn{1}{c}{\textbf{Scenario}} & \multicolumn{1}{c}{\textbf{Initial States}} & \textbf{Ref. V. of EV} & \textbf{Ref. V. of SV} \\
\midrule
Encounter & \makecell[l]{EV: $2.50 \ {\rm m}$, $0 \ {\rm m}$, $\pi/2 \ {\rm rad}$, $0 \ {\rm m/s}$ \\ SV: $-2.50 \ {\rm m}$, $0 \ {\rm m}$, $\pi/2 \ {\rm rad}$, $0 \ {\rm m/s}$} & $0.45 \ {\rm m/s}$& $0.45 \ {\rm m/s}$ \\ \specialrule{0em}{1pt}{1pt}
\hline
\specialrule{0em}{1pt}{1pt}
Overtaking & \makecell[l]{EV: $2.50 \ {\rm m}$, $-0.40 \ {\rm m}$, $\pi/2 \ {\rm rad}$, $0 \ {\rm m/s}$ \\ SV: $1.60 \ {\rm m}$, $2.50 \ {\rm m}$, $\pi \ {\rm rad}$, $0 \ {\rm m/s}$} & $0.45 \ {\rm m/s}$& $0.25 \ {\rm m/s}$ \\ \specialrule{0em}{1pt}{1pt}
\hline
\specialrule{0em}{1pt}{1pt}
Evasion & \makecell[l]{EV: $1.50 \ {\rm m}$, $2.50 \ {\rm m}$, $\pi \ {\rm rad}$, $0 \ {\rm m/s}$ \\ SV: $2.50 \ {\rm m}$, $0.90 \ {\rm m}$, $\pi/2 \ {\rm rad}$, $0 \ {\rm m/s}$} & $0.35 \ {\rm m/s}$& $0.70 \ {\rm m/s}$ \\ \specialrule{0em}{1pt}{1pt}
\bottomrule
\end{tabular}
\begin{tablenotes}\tiny
\item ${\star}$ The initial states from left to right are longitudinal and lateral positions, heading angle, and velocity. 
\item ${\star}$ The reference velocity of the SV is unknown to the EV.
\end{tablenotes} 
\end{table}

\begin{figure}[!t]
\centering
\subfloat[ ]{\includegraphics[width=0.9\columnwidth]{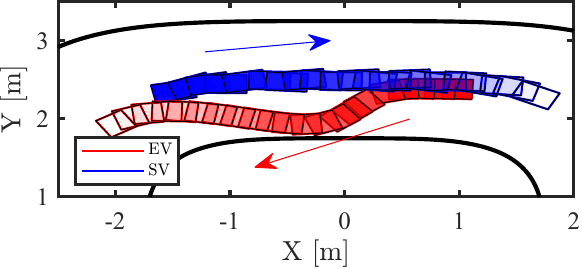}}
\vfil
\subfloat[ ]{\includegraphics[width=0.9\columnwidth]{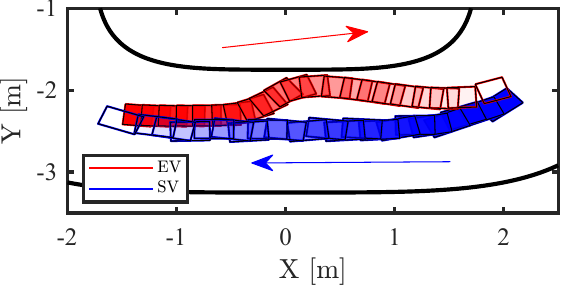}}
\centering
\caption{Measured trajectories of EV and SV during the encounter scenario. (a) The first encounter (from $8.28 \ {\rm s}$ to $15.42 \ {\rm s}$). (b) The second encounter (from $26.33 \ {\rm s}$ to $33.92 \ {\rm s}$).}
\label{fig: EXP1_Encounter_Path}
\end{figure}

\begin{figure}[!t]
\centering
\subfloat[ ]{\includegraphics[width=0.7\columnwidth]{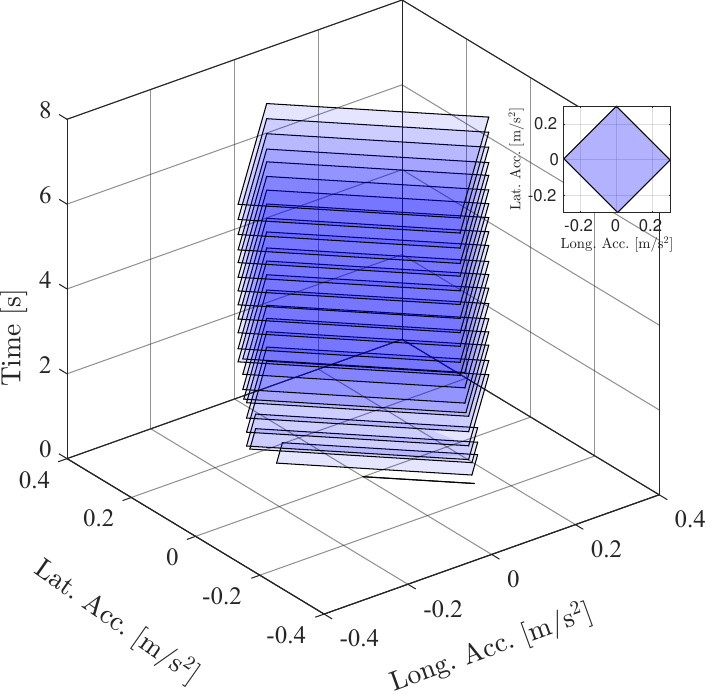}}
\vfil
\subfloat[ ]{\includegraphics[width=0.7\columnwidth]{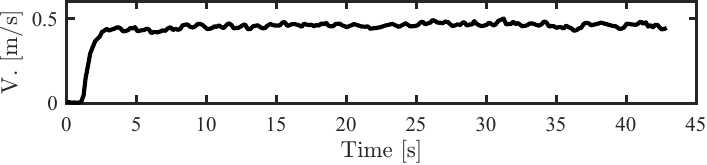}}
\centering
\caption{Information about SV during the encounter scenario. (a) The learned set $\hat{\mathbb{U}}^s_t$ over time in the encounter experiment (the subfigure indicates the final set $\hat{\mathbb{U}}^s_t$). (b)The velocity of SV in the ground coordinate system.}
\label{fig: Set_with_time_encounter}
\end{figure}

\begin{figure*}[!t]
\centering
\subfloat{\includegraphics[height=0.4\columnwidth]{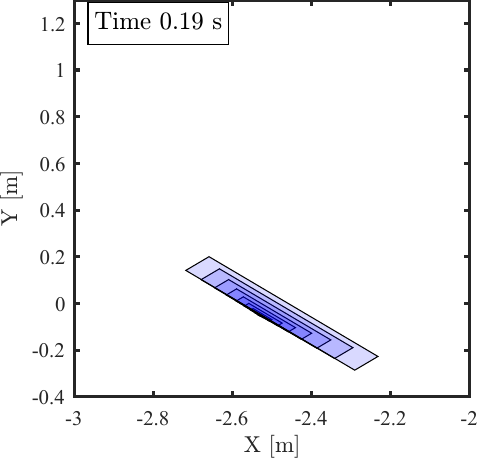}}
\subfloat{\includegraphics[height=0.4\columnwidth]{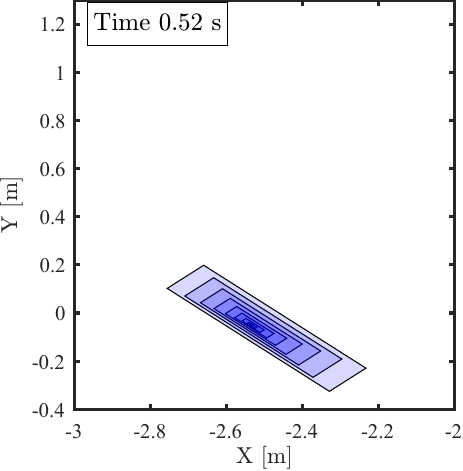}}
\subfloat{\includegraphics[height=0.4\columnwidth]{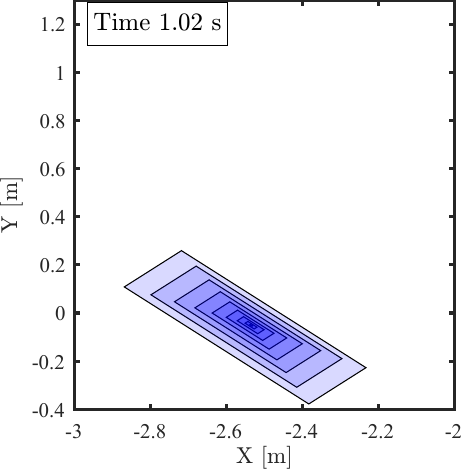}}
\subfloat{\includegraphics[height=0.4\columnwidth]{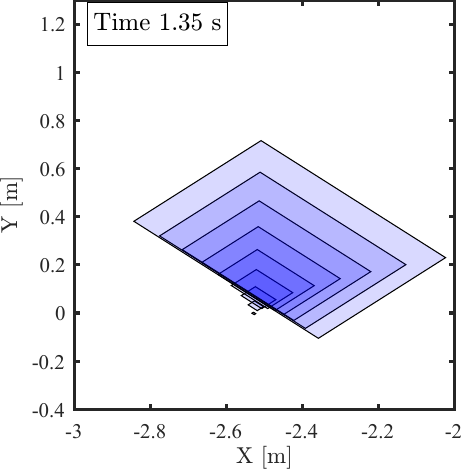}} 
\subfloat{\includegraphics[height=0.4\columnwidth]{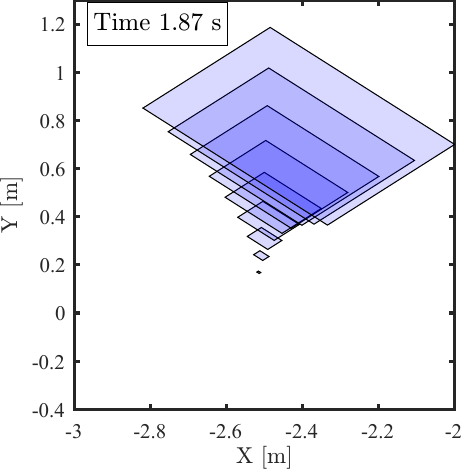}} 
\centering
\caption{Visualization of the predicted occupancy $\hat{\mathcal{O}}_{i|t}^s$ at the initial moments in the encounter scenario.}
\label{fig: EXP_Encounter_Key_Moments_Occ}
\end{figure*}

\begin{figure}[!t]
\centering
\subfloat{\includegraphics[height=0.4\columnwidth]{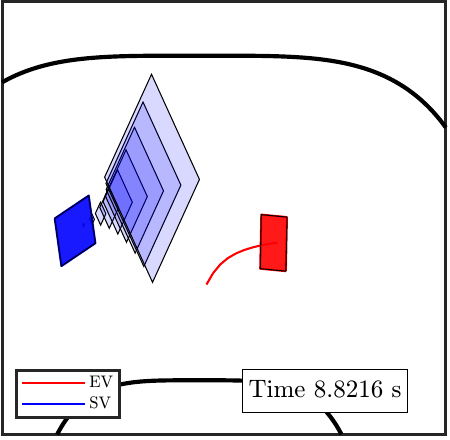}}
\subfloat{\includegraphics[height=0.4\columnwidth]{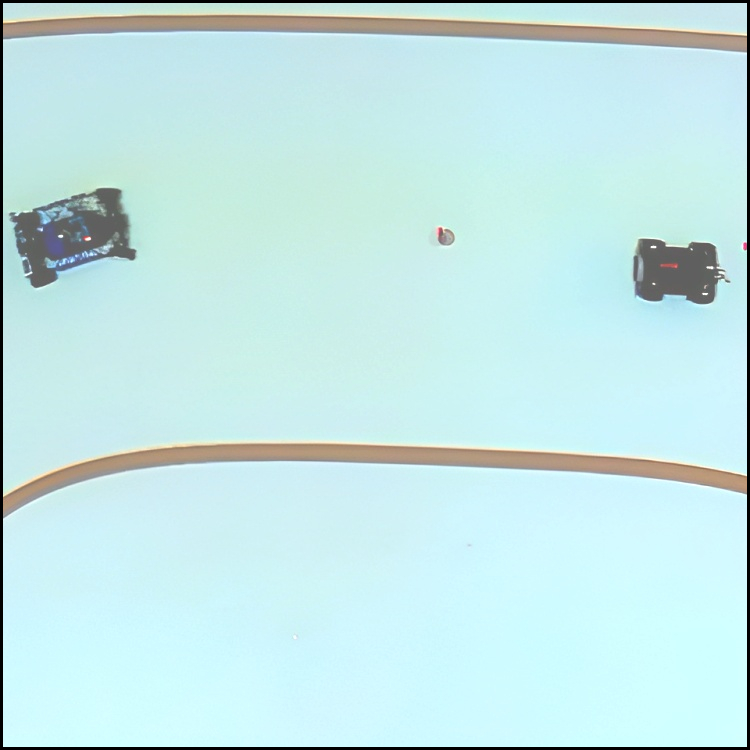}} \\
\subfloat{\includegraphics[height=0.4\columnwidth]{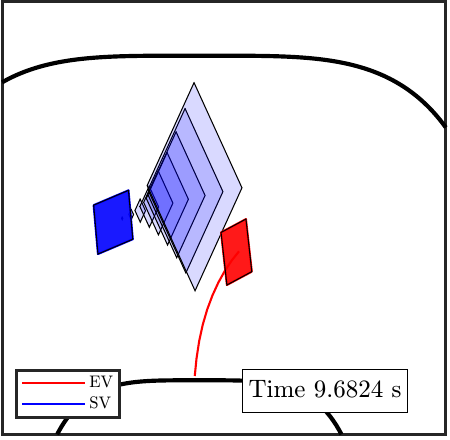}}
\subfloat{\includegraphics[height=0.4\columnwidth]{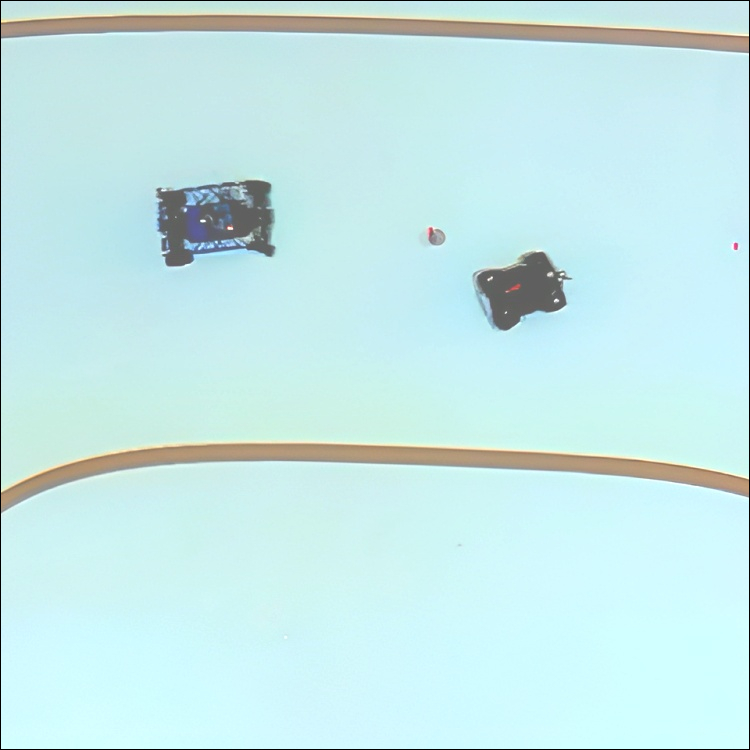}} \\
\subfloat{\includegraphics[height=0.4\columnwidth]{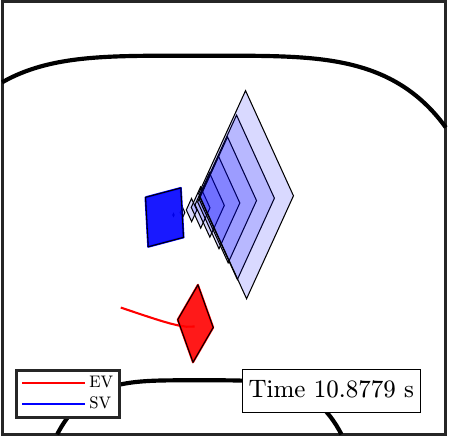}}
\subfloat{\includegraphics[height=0.4\columnwidth]{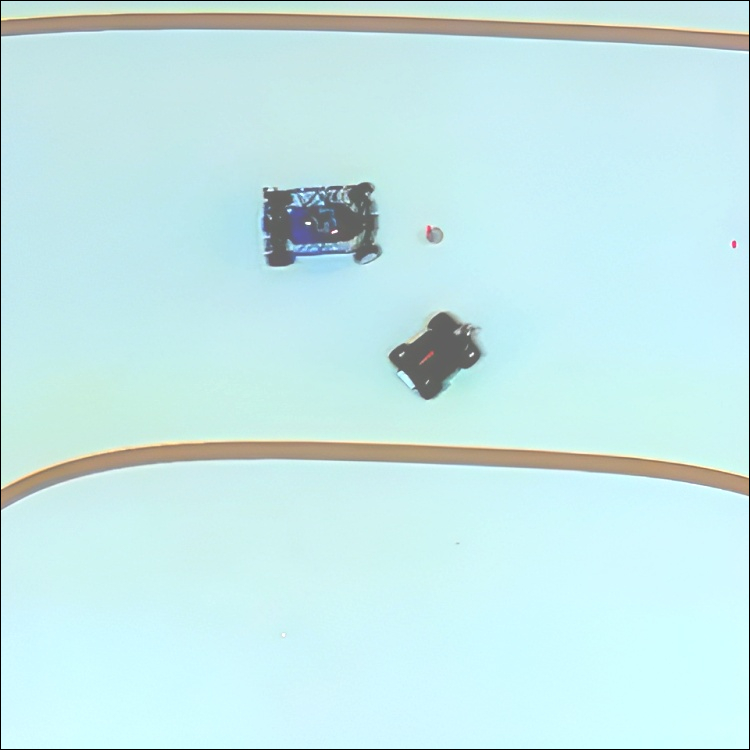}} \\
\subfloat{\includegraphics[height=0.4\columnwidth]{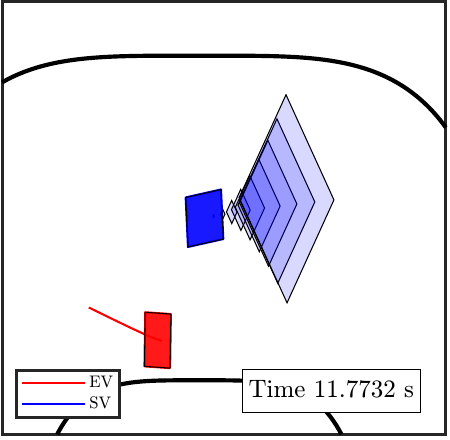}} 
\subfloat{\includegraphics[height=0.4\columnwidth]{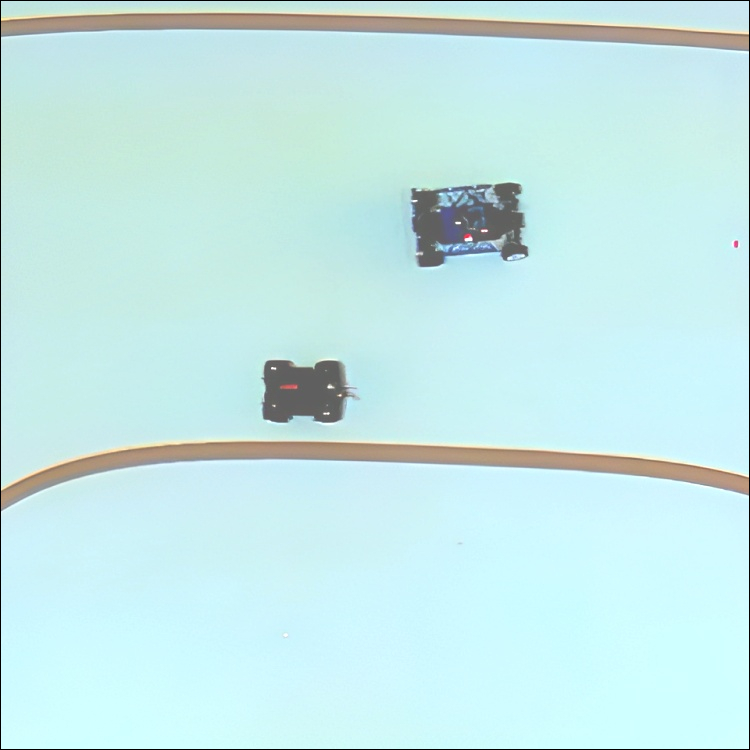}}
\centering
\caption{Visualization of the predicted occupancy of the SV and the planned path of the EV at key moments during the encounter (left column), and the corresponding scenarios in the experiment (right column).}
\label{fig: EXP_Encounter_Key_Moments}
\end{figure}

\begin{figure}[!t]
\centering
\includegraphics[width=0.9\columnwidth]{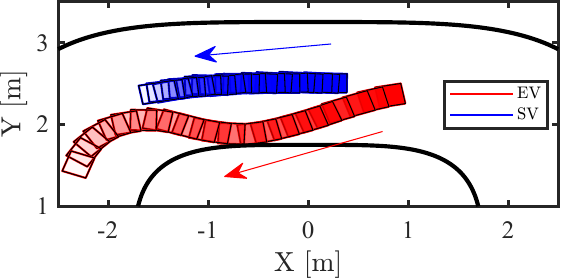}
\caption{Measured trajectories of EV and SV during the overtaking experiment (from $9.56 \ {\rm s}$ to $17.50 \ {\rm s}$).}
\label{fig: EXP2_Overtake_Path}
\end{figure}

\begin{figure}[!t]
\centering
\includegraphics[width=0.9\columnwidth]{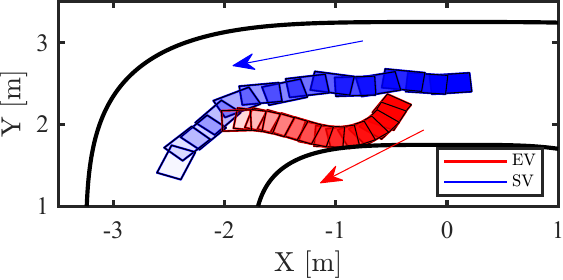}
\caption{Measured trajectories of EV and SV during the evasion experiment (from $6.65 \ {\rm s}$ to $11.27 \ {\rm s}$).}
\label{fig: EXP3_ChaseAvoid_Path}
\end{figure}

This section presents hardware-experiment results in scenarios different from those in Section~\ref{sec: simulations} to demonstrate the method's applicability across different situations, as well as the real-time performance of the method in applications. The experiment videos are accessible online\footnote{https://youtu.be/ulBvjGxYkIE}.

\subsection{Scenario Description}\label{sec: scenario description exp}
The experimental environment is illustrated in Fig.~\ref{fig:two_vehicle_interaction}, where the EV and SV are represented by a hamster platform and a cascar platform, respectively. The size parameters of the EV and SV are the same as those in Table~\ref{Table_exp_parameters robotic}. Both platforms are kinematically equivalent to a full-scale vehicle and have similar dynamic characteristics, including suspension and tire systems, as those of a regular vehicle. This makes them suitable platforms for testing the proposed method in realistic scenarios. In the experiments, the mission of the SV is to move along the center line of the track with a predefined reference velocity, and the objective of the EV is to track the center line of the track and a predefined reference velocity while ensuring collision avoidance with the SV. The SV holds a higher priority than the EV, such that the EV needs to apply the proposed method to replan the reference trajectory at every time step. The SV takes control inputs from an unknown control set $\tilde{\mathbb{U}}_t^s$. The motion-planning performance of the EV is evaluated in three different scenarios:
\begin{itemize}
\item \textbf{Encounter:} EV and SV move in opposite directions, with the EV avoiding collisions with the SV.
\item \textbf{Overtaking:} the faster rear EV surpasses a slower preceding SV.
\item \textbf{Evasion:} the slower preceding EV takes evasive action to avoid a collision with a faster rear SV.
\end{itemize}

The design of the MPC planner follows the same approach as presented in Section~\ref{sec: simulations}, so
details are omitted for conciseness. Note that to enhance the efficiency of solving the OCP problem, several adaptations are applied. First, the dimension of the EV model in \eqref{eq: EV model} is reduced to $4$ by taking the longitudinal acceleration and front tire angle as the inputs. Second, the prediction horizon is changed to $N=9$. 

\subsection{Experiment Results and Discussion} \label{sec: experiment results and discussions}
In the experiments, the polytope computation and the optimization problem are solved by the same tools as in Section~\ref{sec: simulations}.  The positioning of both EV and SV was performed by the Qualisys motion-capture system \cite{qualisys}. The estimation uncertainties of elements in the information set~\eqref{eq_update_information} arose from the state-measurement uncertainties of the SV. However, these measurement uncertainties were minor as a result of the high accuracy of the positioning system. The control inputs of the vehicles were computed on a laptop running Ubuntu 20.04 LTS operating system. The communications were established by the Robot Operating System (ROS Noetic). The reference paths of the EV and SV were followed by a pure-pursuit controller \cite{coulter1992implementation}, and the reference velocities were followed by the embedded velocity-tracking controllers of the hamster platform and the cascar platform, respectively. The initial states and reference velocities of the EV and SV in each scenario are collected in Table~\ref{Table_spe_exp_parameters robotic}. The reference velocity for the hamster platform is designed following the specifications detailed in \cite{berntorp2020positive} and \cite{berntorp2019motion}. In the experiments, the SV was designed to execute non-smooth trajectories by oscillating the front wheels to increase motion uncertainties for the EV. The set $\hat{\mathbb{U}}_t^s$ was learned by estimating the accelerations of the SV in the ground coordinate system, and this does not require oscillation of the SV input.

In the encounter scenario, the EV and SV will intersect twice. The trajectories of both vehicles during these encounters are depicted in Fig.~\ref{fig: EXP1_Encounter_Path}. Fig.~\ref{fig: Set_with_time_encounter}(a) visualizes the learned control set $\hat{\mathbb{U}}_t^s$ of the SV over time, and Fig.~\ref{fig: Set_with_time_encounter}(b) presents the velocity of the SV in the ground coordinate system. Following Fig.~\ref{fig: Set_with_time_encounter}(a), Fig.~\ref{fig: EXP_Encounter_Key_Moments_Occ} illustrates how the predicted occupancy of the SV over the prediction horizon, i.e., $\hat{\mathcal{O}}_{i|t}^s$, evolves at the initial moments. Leveraging the occupancy $\hat{\mathcal{O}}_{i|t}^s$, the EV plans its reference by solving problem \eqref{eq: general ocp}. The occupancy $\hat{\mathcal{O}}_{i|t}^s$, along with the planned reference path of the EV at key moments during the first encounter, are exhibited in Fig.~\ref{fig: EXP_Encounter_Key_Moments}.

In the encounter experiments illustrated in Fig.~\ref{fig: EXP1_Encounter_Path}, the EV successfully avoids a collision with the SV. Fig.~\ref{fig: Set_with_time_encounter}(a) demonstrates that the set $\hat{\mathbb{U}}^s_t$ expands over time as more information is obtained, ultimately converging after approximately $4 \ {\rm s}$. This observation aligns with the SV's velocity, as seen in Fig.~\ref{fig: Set_with_time_encounter}(b), which stabilizes after $4 \ {\rm s}$. Fig.~\ref{fig: EXP_Encounter_Key_Moments_Occ} further illustrates the impact of the set $\hat{\mathbb{U}}^s_t$ on the predicted occupancy $\hat{\mathcal{O}}_{i|t}^s$, which increases at the initial moments and subsequently tends to stabilize when $\hat{\mathbb{U}}^s_t$ ceases to expand. This reflects that in the online execution, the robustness of the method strengthens as the information set $\mathcal{I}_t^s$ becomes richer. Finally, the results presented in Fig.~\ref{fig: EXP_Encounter_Key_Moments} affirm that the EV can effectively plan collision-avoidance reference trajectories in the presence of the uncertain dynamic obstacle, and move back to the track when the obstacle no longer has the risk of collision. 

The trajectories of the EV in the overtaking and evasion experiments are illustrated in Fig.\ref{fig: EXP2_Overtake_Path} and Fig.\ref{fig: EXP3_ChaseAvoid_Path}, respectively. In both scenarios, the proposed method successfully addresses the motion-planning challenges. The corresponding learned sets $\hat{\mathbb{U}}_t^s$ and the predicted occupancies $\hat{\mathcal{O}}_{i|t}^s$, along with the planned collision-avoidance trajectories of the EV, exhibit qualitative similarities. In the presented experiments, the average execution time of the proposed method is $0.064 \ {\rm s}$, and the standard deviation is $0.014 \ {\rm s}$, with the motion planner executed $200$ times in the encounter scenario. The computation time in the overtaking and evasion scenarios is comparably similar. Both the simulations and experiments demonstrate the method's effectiveness in addressing environmental uncertainties and provide validation its real-time applicability. It is also crucial to note that in both Sections~\ref{sec: simulations}--\ref{sec: hardware exp}, the SV is predicted by the double-integrator model~\eqref{eq_SV_model} with the true dynamics completely unknown to the EV, this demonstrated the effectiveness of the method in the presence of unknown dynamics of obstacles.

\subsection{Discussion on Applicability and Future Research}
The case studies used one SV to illustrate the implementation of the proposed method, while the proposed method can directly be applied to manage more obstacles. In addition, the method applies to more scenarios, like highway overtaking, merging, and roundabout driving scenarios. The method's effectiveness has been evaluated in four additional scenarios, including a multiple obstacles scenario, a high-speed overtaking scenario, a cornered track scenario, and a roundabout traffic scenario. The results are presented as supplementary material in the video linked in Section~\ref{sec: experiment results and discussions}. Furthermore, we remark that the proposed method does not have assumptions about the obstacle's control strategy, this means the method also handles the case where the obstacle actively reacts to the maneuver of the EV. This has been verified in an oncoming traffic scenario as in \cite[Section VI. A]{liu2023safe}, where both the EV and obstacle successfully generate collision-free trajectories when they apply the same strategy. The results are also provided as supplementary material in the video. Note that in extreme high-speed high-curvature scenarios, the effectiveness of this method might be affected because of the unpredictable road geometry that can alter the intended trajectories of the obstacles. These factors are not included in this paper.

Future research would focus on integrating with state-of-the-art environment-aware motion-prediction methods, e.g., the map-informed approach~\cite{hallgarten2024stay}, to address problems in more challenging high-speed high-curvature scenarios. Another research direction of great interest is integrating an interaction-aware motion-prediction method, e.g., from a game theoretical perspective~\cite{peters2024contingency}, with the proposed approach for proactive motion planning in multi-vehicle interaction scenarios.

\section{Conclusion} \label{Conclusion}
This paper studied safe motion planning of autonomous robotic systems under obstacle uncertainties.  An online learning-based method was proposed to learn the uncertain control set of the obstacles to compute the forward reachable sets, which were integrated into a robust MPC planner to obtain the optimal reference trajectory for the ego system. Simulations and hardware experiments with an application of the method to an autonomous driving system in different scenarios show that: (1) The method is safer than deterministic MPC and less conservative than worst-case robust MPC, while maintaining safety in uncertain traffic environments;  (2) The method can perform safe motion planning in uncertain environments without prior knowledge of uncertainty information of the obstacles; (3) The method is applicable and real-time implementable in practical scenarios. 

\section*{Acknowledgments}
The authors would like to acknowledge Prof. Lars Nielsen at Link\"oping University for discussions regarding the second-order integrator and polytope-based collision-avoidance constraints. Dr. Theodor Westny at Link\"oping University is acknowledged for skillful assistance in the experiments.

\bibliographystyle{IEEEtran}
\bibliography{IEEEabrv, mybib}


\vspace{-33pt}
\begin{IEEEbiography}[{\includegraphics[width=1in,height=1.25in,clip,keepaspectratio]{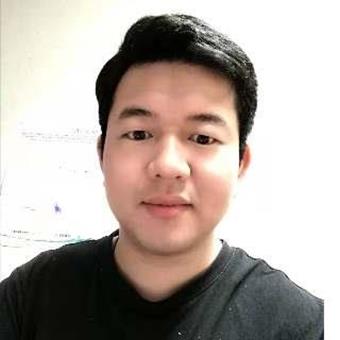}}]{Jian Zhou}
received the B.E. degree in vehicle engineering from
the Harbin Institute of Technology, China, in
2017, and the M.E. degree in vehicle engineering from Jilin University, China, in 2020. He is currently a Ph.D. student with the Department of Electrical Engineering, Link\"oping University, Sweden. His
research interests are motion planning and control for autonomous vehicles and optimization with application to autonomous driving.
\end{IEEEbiography}
\vspace{-20pt}
\begin{IEEEbiography}[{\includegraphics[width=1in,height=1.25in,clip,keepaspectratio]{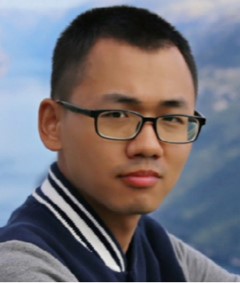}}]{Yulong Gao}
received the B.E. degree in Automation in 2013, the M.E. degree in Control Science and Engineering in 2016, both from Beijing Institute of Technology, and the joint Ph.D. degree in Electrical Engineering in 2021 from KTH Royal Institute of Technology and Nanyang Technological University. He was a Researcher at KTH from 2021 to 2022 and a postdoctoral researcher at Oxford from 2022 to 2023. He is a Lecturer (Assistant Professor) at the Department of Electrical and Electronic Engineering, Imperial College London, from 2024. His research interests include formal verification and control, machine learning, and applications to safety-critical systems.
\end{IEEEbiography}
\vspace{-20pt}
\begin{IEEEbiography}[{\includegraphics[width=1in,height=1.25in,clip,keepaspectratio]{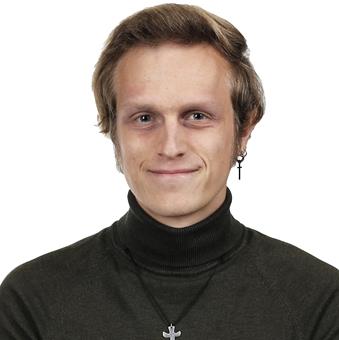}}]{Ola Johansson} received the B.Sc degree in Electrical Engineering in 2020 and the M.Sc degree in Systems Control and Robotics in 2022, both from KTH Royal Institute of Technology, Sweden. He is currently a Research Engineer at the Department of Electrical Engineering, Link\"oping University, Sweden. His research includes positioning, motion planning, and control of drones and robotic systems.
\end{IEEEbiography}
\vspace{-20pt}
\begin{IEEEbiography}[{\includegraphics[width=1in,height=1.25in,clip,keepaspectratio]{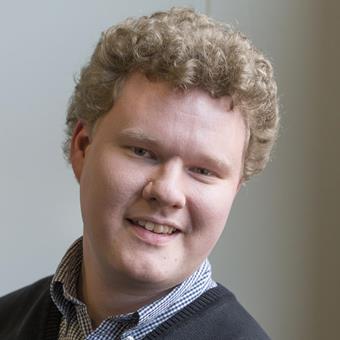}}]{Bj\"orn Olofsson}
received the M.Sc. degree in Engineering Physics in 2010 and the Ph.D. degree
in Automatic Control in 2015, both from Lund University, Sweden. He is currently an Associate Professor at the Department of Automatic Control, Lund University, Sweden, and also affiliated with the Department of Electrical Engineering, Link\"oping University, Sweden. His research includes motion control for robots and vehicles, optimal control, system identification, and statistical sensor fusion.
\end{IEEEbiography}
\vspace{-20pt}
\begin{IEEEbiography}[{\includegraphics[width=1in,height=1.25in,clip,keepaspectratio]{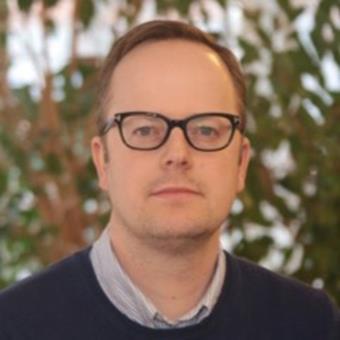}}]{Erik Frisk}
was born in Stockholm, Sweden, in 1971. He received the Ph.D. degree in Electrical Engineering in 2001 from Link\"oping University, Sweden. He is currently a Professor with the Department of Electrical Engineering, Link\"oping University, Sweden. His main research interests are optimization techniques for autonomous vehicles in complex traffic scenarios and model and data-driven fault diagnostics and prognostics. 
\end{IEEEbiography}

\vfill

\end{document}